\documentclass[11pt]{article}

\PassOptionsToPackage{numbers, compress}{natbib}
\usepackage{amsmath,amssymb,amsthm,fullpage,mathrsfs,pgf,tikz,caption,subcaption,mathtools,mathabx}
\usepackage[ruled,vlined,linesnumbered]{algorithm2e}
\usepackage{natbib}
\usepackage{amsmath,amssymb,amsthm,mathtools}
\usepackage[utf8]{inputenc} % allow utf-8 input
\usepackage[T1]{fontenc}    % use 8-bit T1 fonts
\usepackage{hyperref}       % hyperlinks
\usepackage{url}            % simple URL typesetting
\usepackage{booktabs}       % professional-quality tables
\usepackage{amsfonts}       % blackboard math symbols
\usepackage{nicefrac}       % compact symbols for 1/2, etc.
\usepackage{microtype}      % microtypography
\usepackage{times}
\usepackage[margin=1in]{geometry}
\usepackage{bbm}
\usepackage{enumitem}
\usepackage{xcolor}
\usepackage{cleveref}
\usepackage{bm}
\usepackage[ruled,vlined,linesnumbered]{algorithm2e}

\newtheorem{theorem}{Theorem}[section]
\newtheorem{corollary}[theorem]{Corollary}

\newtheorem{lemma}[theorem]{Lemma}
\newtheorem{definition}[theorem]{Definition}

\DeclareMathOperator{\poly}{poly}

\let\oldnl\nl
\newcommand{\nonl}{\renewcommand{\nl}{\let\nl\oldnl}}
\begin{document}
\title{The Role of Randomness in Stability}
\author{Max Hopkins\thanks{Princeton University. Email: \texttt{mh4067@princeton.edu}} \and Shay Moran\thanks{Technion \& Google Research.
Email: \texttt{smoran@technion.ac.il}}}

\maketitle

\begin{abstract}
Stability is a central property in learning and statistics promising the output of an algorithm $\mathcal{A}$ does not change substantially when applied to similar datasets $S$ and $S'$. It is an elementary fact that any sufficiently stable algorithm (e.g.\ one returning the same result with high probability, satisfying privacy guarantees, etc.) must be randomized. This raises a natural question: can we quantify \textit{how much} randomness is needed for algorithmic stability?

We study the randomness complexity of two influential notions of stability in learning: \textit{replicability}, which promises $\mathcal{A}$ usually outputs the same result when run over samples from the same distribution (and shared random coins), and \textit{differential privacy}, which promises the output distribution of $\mathcal{A}$ remains similar under neighboring datasets. The randomness complexity of these notions was studied recently in (Dixon, Pavan, Vander Woude, and Vinodchandran ICML 2024) and (Cannone, Su, and Vadhan ITCS 2024) for basic $d$-dimensional tasks (e.g. estimating the bias of $d$ coins), but little is known about the measures more generally or in complex settings like classification. 

Toward this end, we prove a `weak-to-strong' boosting theorem for stability: the randomness complexity of a task $\mathcal{M}$ (either under replicability or DP) is tightly controlled by the best replication probability of any \textit{deterministic} algorithm solving $\mathcal{M}$, a weak measure called $\mathcal{M}$'s `global stability' that is universally capped at $\frac{1}{2}$ (Chase, Moran, Yehudayoff FOCS 2023). Using this connection, we characterize the randomness complexity of PAC Learning: a class has bounded randomness complexity iff it has finite Littlestone dimension, and moreover scales at worst logarithmically in the excess error of the learner. This resolves a question of (Chase, Chornomaz, Moran, and Yehudayoff STOC 2024) who asked for such a characterization in the equivalent language of (error-dependent) `list-replicability'.
\end{abstract}
\section{Introduction}
Stability is a central property in learning and statistics promising the output of an algorithm $\mathcal{A}$ remains similar when applied to similar datasets $S$ and $S'$. It is an elementary fact that any sufficiently stable algorithm (e.g.\ one returning the same result with high probability, or one satisfying differential privacy) is either randomized, or constant.\footnote{Why? Roughly, if one can interpolate between any two $S$ and $S'$ via a path of `similar' samples, then any deterministic algorithm must give the same result throughout the entire path, and therefore on $S$ and $S'$ as well.} This raises a natural question: \textit{how much} randomness is needed for algorithmic stability? 

In this work, we study the role of randomness in two fundamental notions of stability: replicability, and differential privacy. Replicability is the core scientific and algorithmic tenet that an experiment, run twice on data from the same underlying distribution, should produce the same output with high probability. Given a statistical task $\mathcal{M}$ and an algorithm for the task $\mathcal{A}$ (e.g. hypothesis selection, classification), one might reasonably define the replicability of $\mathcal{A}$ to be its (worst-case) collision probability over independent samples, that is:
\[
\min_{D}\left\{\Pr_{S \sim D^n}[\mathcal{A}(S) = \mathcal{A}(S')]\right\}
\]
where $D$ ranges over the possible data distributions of $\mathcal{M}$. To achieve true replicability, we'd hope to construct an algorithm with collision probability near $1$. Unfortunately, it turns out this is unachievable. The best such parameter, called the \textit{global stability} of $\mathcal{M}$ \cite{bun2020equivalence}, is universally capped at $\frac{1}{2}$ for any non-trivial statistical task \cite{chase2023replicability}. Motivated by this fact,\footnote{More accurately, motivated by a special case of this fact known at the time.} \cite{ImpLPS22} recently proposed an elegant relaxation of global stability allowing \textit{shared internal randomness}, calling an algorithm $\mathcal{A}$ \textit{$\rho$-replicable} if over independent samples $S,S'$ and a shared random string $r$:
\[
\forall D: \Pr_{S,S' \sim D^n,r}[\mathcal{A}(S;r)=\mathcal{A}(S';r)] > 1-\rho.
\]
In \cite{ImpLPS22}'s framework, global stability then exactly measures the best replicability parameter achieved by any \textit{deterministic} algorithm solving $\mathcal{M}$.\footnote{Traditionally a globally stable algorithm isn't assumed to be deterministic, but one can always de-randomize such an algorithm and achieve the same collision probability. See \Cref{sec:global}.}

In \cite{dixon2023list}, the authors introduce \textit{certificate complexity}, the smallest number of random bits needed to achieve $\rho$-replicability. They prove tight bounds on the certificate complexity of several basic $d$-dimensional tasks, as well as showing near-matching bounds on the global stability.\footnote{Formally, \cite{dixon2023list} study a slightly different parameter called list-replicability, which is essentially equivalent to global stability \cite{chase2023replicability}.} Intuitively, it is reasonable to think there should be a general connection between the certificate complexity of a task and its global stability --- the better the global stability, the better replication we can achieve using \textit{no} random bits, so the easier it should be to achieve any particular replication threshold $\rho$. Our first main result confirms this intuition, proving a sort of `weak-to-strong' boosting theorem for replicability: the number of random bits needed to beat $\frac{1}{2}$-replication probability is (up to a single bit) exactly inverse log of the global stability. Furthermore, as one might suspect, this can be amplified to any $\rho$-replicability using an additional $\log(1/\rho)$ random bits.

While replicability and global stability are important notions in their own right, they have perhaps been most impactful in their close relation to the widely influential notion of \textit{differential privacy} \cite{dwork2006calibrating}. An algorithm for a statistical task is called differentially private if its output distribution remains similar under any two neighboring datasets. Like replicability, differentially private algorithms are inherently randomized, leading to the analogous notion of \textit{DP complexity} \cite{canonne2024randomness} measuring the smallest number of internal random bits required to achieve privacy. Our second main result is an analogous `weak-to-strong' boosting theorem for differential privacy, closely tying DP Complexity to a task's underlying global stability. However, due to privacy's multiple parameters and nuanced dependence on sample complexity, our results in this context cannot be said to give an exact equivalence as above, and it remains an interesting open problem to fully characterize the relation between them.

Boosting in hand, we turn to look at the randomness complexity of one of the best studied tasks in learning: binary classification. We focus on the classical PAC model \cite{vapnik1974theory,valiant1984theory}, where a learner, given sample access to a distribution of labeled examples~$D$, must produce a labeling from some fixed class $H$ that is close to the best possible option with high probability. The global stability of PAC learning is quite well studied \cite{bun2020equivalence,ghazi2021sample,bun2023stability,chase2023replicability}. In the `realizable setting', where the data is assumed to be consistent with some $h \in H$, it is known how to construct a $2^{-2^{O(d)}}$-globally stable algorithm \textit{independent of the learner's error}, for $d$ the Littlestone dimension of the class. Recently, \cite{chase2024local} proved this result does not extend to the general `agnostic' setting and ask whether it is instead possible to prove a bound that decays with the excess error $\alpha$. We resolve this problem: a class $H$ has bounded error-dependent global stability (and therefore also certificate complexity) if and only if it has finite Littlestone dimension. Moreover, the randomness complexity suffers only mild dependence on the error, scaling at worst as $O(\log\frac{1}{\alpha})$.

\subsection{Main Results}
We briefly overview our main setting of study. A \textbf{statistical task} $\mathcal{M}$ consists of a \textit{data domain} $\mathcal{X}$, an \textit{output domain} $\mathcal{Y}$, and, for every distribution $D$ over $\mathcal{X}$, a family $G_D \subset \mathcal{Y}$ of `accepted solutions' for the task. An algorithm $\mathcal{A}: \mathcal{X}^* \to \mathcal{Y}$ `solves' the task $\mathcal{M}$ if for any $\beta>0$, given sufficiently many samples $n(\beta)$ from any $D$ over $\mathcal{X}$, $\mathcal{A}$ outputs $y \in G_D$ with probability at least $1-\beta$. At a parametrized level, for fixed $\beta$ and $n(\beta)$, we say the algorithm solves $\mathcal{M}$ with \textit{confidence} $\beta$ and \textit{sample complexity} $n(\beta)$.

Recall an algorithm $\mathcal{A}$ is called $\rho$-\textbf{replicable} if over two independent samples and a shared random string, it returns the same output with high probability:
\[
\forall D: \Pr_{r,S,S' \sim D^n}[\mathcal{A}(S;r)=\mathcal{A}(S';r)] \geq 1-\rho.
\]
The (parameter-free) \textbf{certificate complexity} of a task $\mathcal{M}$, denoted $C_{\text{Rep}}$, is the smallest $\ell$ such that there exists a better than $\frac{1}{2}$-replicable algorithm solving $\mathcal{M}$, or, more explicitly, the smallest $\ell$ such that for every $\beta>0$ there exists $n(\beta) \in \mathbb{N}$ and a better than $\frac{1}{2}$-replicable algorithm on $n(\beta)$-samples and $\ell$-random bits solving $\mathcal{M}$ with confidence $\beta$.

The \textbf{global stability} of a task $\mathcal{M}$ is the best replicability that can be achieved using \textit{no} internal randomness. In particular, an algorithm $\mathcal{A}$ is called $\eta$-globally stable if over independent samples $S,S'$:
\[
\forall D: \Pr_{S,S' \sim D^n}[\mathcal{A}(S)=\mathcal{A}(S')] \geq \eta.
\]

The \textbf{globally-stable complexity} (g-stable complexity), denoted $C_{\text{Glob}}$, is $\log\frac{1}{\eta_{\mathcal{M}}}$, where $\eta_{\mathcal{M}}>0$ is the supremum across $\eta$ for which there is a deterministic $\eta$-globally stable algorithm solving $\mathcal{M}$. Note that any $\eta$-globally stable algorithm automatically has an output which is an \textbf{$\eta$-heavy-hitter}, that is some $y \in \mathcal{Y}$ for which $\Pr[\mathcal{A}(S)=y] \geq \eta$. The traditional notion of global stability \cite{bun2020equivalence} only requires this latter condition. In \Cref{sec:global} we show our variant is equivalent up to minor differences in sample complexity (meaning in particular the g-stable complexity \textit{does not depend} on which definition you take).

With this in mind, our first main result is a weak-to-strong boosting lemma for replicability: the number of random bits needed to achieve high probability replicability is exactly controlled by the best replication probability of any deterministic algorithm, i.e.\ g-stable and certificate complexity are (essentially) equivalent.
\begin{theorem}[Stability vs Replicability (\Cref{thm:equiv})]\label{thm:intro-parameter-free}
    Let $\mathcal{M}$ be any statistical task. Then:
    \[
    C_{\text{Glob}} \leq C_{\text{Rep}} \leq C_{\text{Glob}}+1.
    \]
    Moreover, the number of random bits required to achieve $\rho$-replicability is at most $\lceil C_{\text{Glob}}+\log(1/\rho) \rceil$.
\end{theorem}
Replicability and global stability have played an important role in recent work constructing efficient differentially private algorithms in learning and statistics \cite{bun2020equivalence,GhaziKM21,bun2023stability,kalavasis2023statistical}). An algorithm $\mathcal{A}$ is said to be $(\varepsilon,\delta)$\textbf{-differentially private} if for any two datasets $S,S' \in \mathcal{X}^n$ differing in only one coordinate, the output distribution of $\mathcal{A}$ is nearly indistinguishable in the following sense. For any measurable event~$\mathcal{O} \subset \mathcal{Y}$:
\[
\Pr[\mathcal{A}(S) \in \mathcal{O}] \leq e^{\varepsilon}\Pr[\mathcal{A}(S') \in \mathcal{O}]+\delta.
\]
In recent work, \cite{canonne2024randomness} introduce the \textit{DP (randomness) complexity} of a task $\mathcal{M}$, denoted $C_{DP}(n,\beta,\varepsilon,\delta)$, measuring the minimum number of random bits required to construct an $(\varepsilon,\delta)$-DP algorithm on $n$ samples solving $\mathcal{M}$ with $\beta$-confidence. Unlike our prior notions, DP complexity is \textit{parametrized} due to the fact that there is no clear `threshold' to set for the privacy parameters $(\varepsilon,\delta)$ as in the replicability setting (i.e.\ $\rho=1/2$). Furthermore, $(\varepsilon,\delta)$-privacy is only meaningful when the parameters are taken as functions of the sample complexity $n$ (and therefore confidence $\beta$), leading to the somewhat cumbersome $C_{DP}(n,\beta,\varepsilon,\delta)$.

As a result of the above, the connection between DP complexity and global stability is somewhat more nuanced than the parameter-free equivalence in \Cref{thm:intro-parameter-free}. We will consider two variants of the connection. In the first, we compare DP complexity to analogously defined parametrized g-stable and certificate complexities $C_{\text{Glob}}(n,\beta)$ and $C_{\text{Rep}}(n,\beta)$ (see \Cref{sec:prelims-DP} for formal definition).
\begin{theorem}[Stability vs DP (Informal \Cref{thm:DP-list})]\label{thm:intro-list-DP}
    There exists a universal constant $c>0$ such that for any statistical task $\mathcal{M}$:
    \begin{enumerate}
        \item \textbf{(Stability to DP):} $C_{DP}(n,\beta,\varepsilon,\delta) \leq C_{\text{Glob}}(n',\beta')+\log(1/\varepsilon)+\log(1/\delta)+\tilde{O}(1)$
        \\
        \\
        \vspace{.2cm}
        % \noindent \hspace{-.8cm}
        for any $n \geq n'\cdot \exp(C_{\text{Glob}}(n',\beta'))\frac{\log(\frac{1}{\delta})\log(\frac{1}{\beta})}{\varepsilon}$ and $\beta \geq \beta' \cdot  \exp(C_{\text{Glob}}(n',\beta'))\frac{\log(\frac{1}{\delta})}{\varepsilon}$
        % \poly(\varepsilon,\log^{-1} \frac{1}{\delta}, \log^{-1} \frac{1}{\beta}, \exp(-C_{\text{Glob}}))$.
        \item \textbf{(DP to Stability):} $C_{\text{Glob}}(n,\beta) \leq C_{DP}(n',\beta',\varepsilon,\delta) + O(1)$,
    \end{enumerate}
    for any $(n',\varepsilon,\delta)$ satisfying $\varepsilon \leq \frac{c}{\sqrt{n'\log(n')}}$ and $\delta \leq \frac{c}{n'}$, $n \geq n' \exp(C_{\text{Glob}}(n,\beta))$, and $\beta' \geq \beta \exp(C_{\text{Glob}}(n,\beta))$
\end{theorem}

Qualitatively, \Cref{thm:intro-list-DP} (Item 1) simply states that given an $\eta$-globally stable (deterministic) algorithm, we can `boost' it into an ($\varepsilon,\delta$)-DP algorithm using only $\log(1/\eta)+\log(1/\varepsilon)+\log(1/\delta)$-random bits, where the new algorithm may have somewhat higher failure probability ($\beta$ vs $\beta'$) and sample complexity ($n$ vs $n'$) than the original. This should be compared to the analogous statement in replicability (\Cref{thm:intro-parameter-free}), which roughly stated an $\eta$-globally stable algorithm can be boosted to a $\rho$-replicable one using $\log(1/\eta)+\log(1/\rho)$ random bits. \Cref{thm:intro-list-DP} (Item 2) states the (weak) converse also holds: given a good enough DP algorithm on $\log(1/\eta)$ random bits, we can `reverse-engineer' a deterministic $\eta$-globally stable algorithm from it, again paying in the sample size and confidence.

We remark the blowup in sample complexity and confidence in \Cref{thm:intro-list-DP} is often mild --- while the exponential dependence on $C_{\text{Glob}}$ looks large, $C_{\text{Glob}}$ is log-scale, so the blowup is only polynomial in stability similar to prior works in the area \cite{bun2020equivalence,bun2023stability}. Nevertheless, the core focus of our work is on one specific computational resource---randomness. As a result, we will sometimes compromise on other resources as we do here to better understand the fundamental limits of the main one.

Nevertheless, since the quantitative statement of \Cref{thm:intro-list-DP} is fairly cumbersome otherwise, it is still worth taking a minute to explain the various dependencies before moving on. Recall our goal: given an $\eta$-globally stable algorithm $\mathcal{A}_{\text{global}}$ on $n'$ samples that succeeds with probability $1-\beta'$, we wish to use $\mathcal{A}_{\text{global}}$ as a subroutine to construct an $(\varepsilon,\delta)$-DP algorithm solving the same problem with slightly worse probability $1-\beta$ and using slightly more $(n)$ samples. As in prior work, the key to doing this is to build a large database of $\mathcal{A}_{\text{global}}$'s \textit{$\eta$-heavy-hitters}, outputs of $\mathcal{A}_{\text{global}}$ that appear with probability at least $\eta$. By standard concentration bounds, this can be done by running $\mathcal{A}_{\text{global}}$ roughly $O\left(\frac{\log \frac{1}{\beta}}{\eta^2}\right)$ times and looking at empirical output probabilities. Outputting one of these heavy-hitters privately actually requires us to run this estimation process $O\left(\frac{\log \frac{1}{\delta}}{\eta\varepsilon}\right)$ times to ensure the resulting `dataset' of heavy hitters isn't too strongly affected by changing any single input sample in the process. In total, the above uses $n = n' \cdot \frac{\log(\frac{1}{\delta})\log(\frac{1}{\beta})}{\eta^3\varepsilon}$ total samples and setting $\eta$ to be the best achievable replication probability ($2^{-C_{\text{Glob}}(n',\beta')}$) gives the stated dependence. The decay of the success probability $\beta$ can similarly be thought of as coming from bounding the probability any individual subroutine $\mathcal{A}_{\text{global}}$ fails, though we give a better bound through more careful analysis in the main body.

Of course the basic procedure described above, which is similar to prior DP-to-stability reductions in the literature \cite{bun2023stability}, is not randomness-efficient (indeed it may even use an unbounded number of random bits!) As we discuss in \Cref{sec:pf-view-DP}, our randomness-efficient variant relies on carefully discretizing this type of transform to optimize the number of random bits without sacrificing correctness and privacy.

In the reverse direction (Item 2), we directly prove any sufficiently private algorithm (i.e.\ one with the stated dependencies on $(n,\varepsilon,\delta)$) that uses $k$ random bits \textit{automatically} has a heavy-hitter of weight roughly $\exp(-k)$. We then show how to transform any $\beta'$-correct randomized algorithm $\mathcal{A}$ on $n'$ samples with an $\eta$-heavy-hitter into a deterministic $\eta$-globally stable algorithm at the cost of running $\mathcal{A}$ roughly $\frac{\log(1/\beta)}{\eta^2}$ times, resulting in similar parameter blowups as in Item 1. We remark that the constraints on $(n,\varepsilon,\delta)$ are also in general fairly mild. The assumption on $\delta$ is extremely weak (an algorithm is not considered private unless $\delta \leq n^{-\omega(1)}$). The assumption on $\varepsilon$ is more restrictive, but is satisfied by many basic DP mechanisms, and, moreover, it is almost always possible to amplify a weak DP algorithm to one satisfying this constraint (though it may cost many additional random bits). We discuss this further in \Cref{sec:discussion}.

% We discuss the requirements on $(n,\varepsilon,\delta)$ further in \Cref{sec:discussion}, but for the moment simply note that the assumption on $\delta$ is extremely weak (typically an algorithm is not considered private unless $\delta \leq n^{-\omega(1)}$), and that while the assumption on $\varepsilon$ is more restrictive, it is satisfied by many popular DP mechanisms and even those that do not can almost always be transformed (albeit at the possible cost of additional randomness scaling with the sample complexity of the algorithm)

Moving on from \Cref{thm:intro-list-DP}, we'd also like a way to compare randomness complexity in DP to our original \textit{parameter-free} variants of $C_{\text{Glob}}$ and $C_{\text{Rep}}$. It turns out this is possible by considering a slight generalization of vanilla differential privacy called \textit{user-level} DP. In user-level DP there are $T$ `users', each of whom contributes a (sub)-dataset $S_i$. Neighboring datasets are defined with respect to swapping an entire user rather than a single example. In other words, in user-level DP we view the total size-$n$ dataset $S$ as being comprised of $T$ components (`users') $(S_1,\ldots,S_{T}) \in (\mathcal{X}^{n/T})^T$, and must maintain $(\varepsilon,\delta)$-privacy under swapping out an entire $S_i$ subsample. Critically, in this setting $(\varepsilon,\delta)$ are now \textit{functions of $T$} rather than of~$n$. This allows us to define the user-level DP complexity, $C_{DP}(T,\varepsilon,\delta)$, in a way that is independent of sample complexity and confidence as the smallest number of random bits such that there exists a $T$-user $(\varepsilon,\delta)$-DP algorithm solving $\mathcal{M}$. We then get the following cleaner `parameter-free' version of \Cref{thm:intro-list-DP}:
\begin{theorem}[Stability vs User-Level DP (Informal \Cref{thm:list-DP-user})]\label{thm:intro-list-DP-user}
    There exist universal constants $c_1,c_2>0$ such that for any statistical task $\mathcal{M}$:
    \begin{enumerate}
        \item \textbf{(Stability to DP):} $C_{DP}\left(2^{C_{\text{Glob}}}\frac{c_1\log \frac{1}{\delta}}{\varepsilon},\varepsilon,\delta\right) \leq C_{\text{Glob}}+\log(1/\varepsilon)+\log(1/\delta)+\tilde{O}(1)$.
        \item \textbf{(DP to Stability):} $C_{\text{Glob}} \leq C_{DP}(T,\varepsilon,\delta) + O(1)$,
    \end{enumerate}
    where the latter holds for any $(T,\varepsilon,\delta)$ satisfying $\varepsilon \leq \frac{c_2}{\sqrt{T\log(T)}}$ and $\delta \leq \frac{c_2}{T}$.
\end{theorem}
We remark that the parameter dependencies here follow exactly the same explanation as in \Cref{thm:intro-list-DP}, with $T=2^{C_{\text{Glob}}}\frac{c_1\log \frac{1}{\delta}}{\varepsilon}$ in (Item 1) appearing due to finding heavy hitters of $\mathcal{A}_{\text{global}}$, and the $(T,\varepsilon,\delta)$ dependency in (Item 2) needed to imply the corresponding DP algorithm has a heavy hitter.

% It is worth noting that there is no asymptotic loss in the conversion back and forth between global stability and DP in \Cref{thm:intro-list-DP-user}. Starting with an optimal globally stable learner, our transform gives a $2^{C_{\text{Glob}}}\frac{\log \frac{1}{\delta}}{\varepsilon}$-user level DP algorithm on $C_{\text{Glob}}+\log(1/\varepsilon)+\log(1/\delta)$ random bits. Transforming back requires setting $\varepsilon,\delta \lesssim 2^{-C_{\text{Glob}}}$, so we get a $3C_{\text{Glob}}$-globally stable learner in return, losing only constant factors.

\subsubsection{The Stable Complexity of PAC Learning} 
Having established our stability boosting theorems, we turn to the randomness complexity of binary classification. We focus on the standard \textbf{PAC Learning} model \cite{valiant1984theory,vapnik1971}. A PAC Learning classification problem consists of a \textit{data domain} $X$ and a \textit{hypothesis class} $H=\{h:X \to \{0,1\}\}$ of potential labelings of $\mathcal{X}$. Given a distribution $D$ over labeled samples $X \times \{0,1\}$, the \textit{classification error} of a hypothesis $h$ is
\[
err_D(h) \coloneqq \Pr_{(x,y) \sim D}[h(x) \neq y].
\]
An algorithm is said to (agnostically) PAC learn the class $(X,H)$ if for every $\alpha,\beta>0$ there exists $n=n(\alpha,\beta) \in \mathbb{N}$ such that given $n$ samples, $\mathcal{A}$ outputs an $\alpha$-optimal hypothesis with probability at least $1-\beta$:
\[
\forall D: \Pr_{S \sim D^n}[err_D(\mathcal{A}(S)) > \min_{h \in H} err_D(h) + \alpha] \leq \beta.
\]
In our framework, PAC Learning can be viewed as a sequence of statistical tasks $\{\mathcal{M}_\alpha\}$ parameterized by the error~$\alpha$, where $\mathcal{X}=X \times \{0,1\}$, $\mathcal{Y}$ is the set of all labelings of $X$, and $G_D$ is the set of $\alpha$-optimal labelings.

Due to its close connection with differential privacy, stability is quite well studied in the PAC setting. In \cite{bun2020equivalence}, the authors show that under the assumption that $\min_{h \in H} err_D(h)=0$ (called the `realizable setting'), the g-stable complexity of $\{\mathcal{M}_\alpha\}$ can be universally upper bounded by $2^{O(d)}$ for any class $(X,H)$ with finite Littlestone dimension $d$. In the error-dependent setting, the best known bound is of \cite{ghazi2021sample,GhaziKM21} who improve the g-stable complexity to $\poly(d)+O(\log(\frac{1}{\alpha}))$.

Contrary to the above, \cite{chase2024local} show the former type of error-independent bound is impossible in the agnostic setting, and ask whether an error-dependent bound like \cite{ghazi2021sample} can be extended to this case.\footnote{Again, we note \cite{chase2024local} is phrased in terms of `list-replicability', but \cite{chase2023replicability} prove the $\alpha$-dependent (log) list-size $C_{\text{List}}(\alpha)$ satisfies $C_{\text{Glob}}(\alpha) \leq C_{\text{List}}(\alpha) \leq C_{\text{Glob}}(\alpha/2)$, so it is asympotically equivalent to characterize $C_{\text{Glob}}(\alpha)$. We refer the reader to Section 5 of \cite{chase2023replicability} and Section 2.3.1 of \cite{chase2024local} for details and the formal definitions.}
We resolve this problem: not only is such a bound is possible, the complexity scales essentially as in the realizable setting up to a factor in the VC dimension.

% In the realizable setting, the list complexity of PAC learning is reasonably well understood \cite{bun2020equivalence,ghazi2021sample}, and in particular by a similar connection to differential privacy it is known that $(X,H)$ has finite list complexity if and only if the class has finite Littlestone dimension (a combinatorial parameter of $H$ characterizing both online and differentially private learning). In the agnostic case, \cite{chase2023replicability} show a surprising `separation' between DP and list complexity -- every infinite class must have list complexity scaling with the excess

% Critically, \cite{chase2023replicability} assume in their result that the list complexity may not depend on the excess error of the algorithm. They ask whether relaxing this constraint allows for finite stable complexity. Given \Cref{thm:intro-list-DP}, one might suspect

% We resolve this problem: not only is the (excess-risk dependent) list complexity finite for finite Littlestone classes, it scales at worst logarithmically in the error $\alpha^{-1}$.

\begin{theorem}[The Certificate Complexity of Agnostic Learning (\Cref{thm:stable-agnostic})]\label{thm:intro-agnostic}
    Let $(X,H)$ be a hypothesis class with Littlestone dimension $d$. Then $(X,H)$ has a better than $\frac{1}{2}$-replicable learner with
    \begin{enumerate}
        \item \textbf{Sample Complexity:} \[
        \exp(\poly(d))\poly(\alpha^{-1},\log(1/\beta))
        \]
        \item \textbf{Certificate Complexity:} 
        \[
        \poly(d)+O(VC(H)\log(\frac{1}{\alpha}))
        \]
    \end{enumerate}
    Conversely if $d=\infty$, then $C_{\text{Rep}}(\alpha)=\infty$ for any $\alpha < \frac{1}{2}$. In other words, there is no globally stable, DP, or replicable algorithm for $(X,H)$ better than random guessing.
\end{theorem}
We remark if all one wants is heavy-hitter global stability, the sample complexity of \Cref{thm:intro-agnostic} can be improved to $\poly(d,\alpha^{-1},\log(1/\beta))$ while maintaining $(\poly(d)+O(VC(H)\log(\frac{1}{\alpha})))$-g-stable complexity. It is an interesting question whether $\poly(d,\alpha^{-1},\log(1/\beta))$ samples is achievable in the replicable case while maintaining good certificate complexity. This cannot be achieved using current sample-efficient methods which all rely on correlated sampling \cite{GhaziKM21,bun2023stability} and have certificate size scaling with $|H|$. Simultaneously achieving randomness and sample efficiency therefore seems to require new ideas in the theory of replicable algorithm design.

In recent independent work, \cite{blonda2025stability} also prove a variant of \Cref{thm:intro-agnostic} for global stability with $g$-stable complexity $\exp(d)+\poly(\alpha^{-1})$. Their work also addresses a second open problem of \cite{chase2024local} which we do not consider. See \Cref{sec:related-work} for further discussion.

\subsection{Discussion and Open Problems}\label{sec:discussion}
Before moving to the formal details, we take some space to further discuss the parameter restrictions in \Cref{thm:intro-list-DP} and \Cref{thm:intro-list-DP-user}, related open problems, and provide some further justification why counting internal random bits is sensible for statistical tasks where `external' randomness is also available via samples.

\paragraph{DP Complexity and the Pareto Frontier:} Recall both \Cref{thm:intro-list-DP} and \Cref{thm:intro-list-DP-user} require a variant of the following guarantee on the privacy parameters $(T,\varepsilon,\delta)$
\begin{equation}\label{eq:DP-Constraints}
\varepsilon \lesssim \frac{1}{\sqrt{T\log(T)}}, \delta \lesssim \frac{1}{T}
\end{equation}
where $T$ is the number of users (in the vanilla DP setting, $T=n$ is then just the number of samples). Our first main question is to what extent such a requirement is necessary. Can we characterize for what triples $(T,\varepsilon,\delta)$ the inequality $C_{\text{Glob}} \leq C_{DP}(T,\varepsilon,\delta)+O(1)$ holds?

Toward this end, it is worth briefly discussing why this constraint occurs, and to what extent it is normally achievable. Regarding the first, \Cref{eq:DP-Constraints} are exactly the constraints required to achieve a strong notion of stability called \textit{perfect generalization} \cite{CummingsLNRW16,BassilyF16}, which promises that on most input datasets $S$, $\mathcal{A}(S;\cdot)$ is actually statistically close to the distribution $\mathcal{A}(\cdot;\cdot)$. All known transforms between DP and replicability go through perfect generalization \cite{bun2023stability,kalavasis2023statistical} (we note prior transforms combine this with a method called correlated sampling which ruins the certificate size, so do not recover our results). At a high level, it would be interesting to give a transform between DP and list-replicability without going through perfect generalization, potentially bypassing this barrier. In some sense such a transformation is actually possible in the PAC Learning setting using the fact that Littlestone dimension characterizes learning, but this results in tower-type dependence in the certificate size.

Regarding the achievability of \Cref{eq:DP-Constraints}, in most cases such constraints are fairly mild. In fact, normally \Cref{eq:DP-Constraints} is hardly a barrier at all since it is always possible to \textit{amplify} a $T$-user $(O(1),1/\poly(T))$-DP algorithm to a $\frac{T}{\varepsilon}$-user $(\varepsilon,1/\poly(T))$-DP algorithm simply by applying the former on a random size-$T$ sub-sample of the latter's users \cite{balle2018privacy}. Setting $\varepsilon \ll \frac{1}{T}$ then satisfies \Cref{eq:DP-Constraints}. The catch here is that this process uses $T\log(1/\varepsilon)$ additional random bits, a heavy cost when known settings seem to only require \textit{logarithmically} many bits in $T$ \cite{canonne2024randomness}. This leads to an interesting question: is it possible to derandomize subsampling techniques in differential privacy? Can we amplify an $(O(1),1/\poly(T))$-DP algorithm to $(\varepsilon,1/\poly(T))$-DP using only $O(\log(T/\varepsilon))$ additional random bits?

\paragraph{Measuring Randomness in Statistics:} Certificate and DP complexity count the number of \textit{internal} random bits used by an algorithm solving a fixed statistical task $\mathcal{M}$. Unlike the empirical setting (as studied e.g.\ in \cite{canonne2024randomness}), algorithms in the statistical setting also have access to \textit{external} randomness in the form of i.i.d samples from the underlying data distribution $D$. It is reasonable to then ask: why do such algorithms need internal randomness at all, and (assuming they do) why does it make sense to measure it?

For certificate complexity, this question was naturally addressed in the seminal works of \cite{ImpLPS22,dixon2023list}. Recall that in replicability, the internal randomness of $\mathcal{A}$ plays a special role apart from the sample because it is \textit{shared} between the two runs of the algorithm. It is not hard to argue that most natural problems (namely any problem where there is a continuous path between two distributions $D$ and $D'$ with disjoint solution sets) require at least one shared random bit to achieve better than $1/2$ replication probability.\footnote{The argument is simply that if the algorithm is correct, it must be outputting different solutions on $D$ and $D'$, but then there is some interior distribution on the path which is non-replicable. See \cite{dixon2023list,ImpLPS22,chase2023replicability} for more detailed examples.} Another natural interpretation of the certificate complexity is as a measure of the `communication complexity' of a statistical task --- it is the number of bits that must be published publicly to ensure external parties can verify the result of the algorithm.

In differential privacy, the situation is more nuanced. In \cite{canonne2024randomness}, the authors focus on empirical tasks, meaning the dataset is fixed (not drawn i.i.d from some distribution) and the only source of randomness is internal. In this case it is obvious that any non-trivial DP algorithm must be randomized, and there is strong motivation to bound the extra number of random bits required since 1) clean randomness is expensive, and 2) prior implementations of DP in practice such as the U.S.\ Census used an astronomical number of clean random bits, thought to be upwards of 90 terabytes \cite{garfinkel2020randomness}. Understanding to what extent this cost can be mitigated is then a question of significant theoretical and practical interest.

In the statistical setting, we argue that despite our additional access to external randomness, any useful notion of differential privacy must still inherently rely only on internal randomness to protect user data. This is because differential privacy should always protect users in the \textit{worst-case}, even for \textit{average-case problems} studied in statistics. In particular, since the algorithm has no control over the sample oracle, relying on its output to ensure privacy leaves it open to attacks corrupting the oracle, or even to privacy failure stemming from benign issues such as light failure of the i.i.d assumption in real-world data (see e.g.\ discussion in \cite{goldwasser2022planting} regarding possible breaches in security from corruption of a learning algorithm's randomness). To avoid such scenarios, when defining differential privacy for statistical tasks we always apply privacy at the level of the \textit{fixed sample}, so the same motivation and context as in the empirical case holds.

\subsection{Further Related Work}\label{sec:related-work}

\paragraph{Stability and Replicability:}
Global stability was introduced in \cite{bun2020equivalence} to upper bound the sample complexity of private PAC Learning. Replicability was introduced independently in \cite{ImpLPS22} and \cite{GhaziKM21} as a relaxation of global stability, in the former towards addressing the crisis of replicability in science, and in the latter towards achieving upper bounds on the sample complexity of user-level private learning. Since these works, a great deal of effort has gone into understanding what statistical tasks admit replicable algorithms \cite{karbasi2023replicability,esfandiari2024replicable,eaton2024replicable,esfandiari2022replicable,komiyama2024replicability,hopkins2024replicability,kalavasis2024replicable}, as well as their connection to other notions of stability such as differential privacy \cite{GhaziKM21,bun2023stability,kalavasis2023statistical,moran2023bayesian,chase2023replicability}. Our work is most related to this latter line, especially \Cref{thm:intro-list-DP} and \Cref{thm:intro-list-DP-user} which rely on new randomness-efficient variants of the transforms introduced in these works.

\paragraph{Certificate and DP Complexity:}
Certificate complexity was introduced in \cite{dixon2023list}, where the authors prove essentially tight bounds for the basic task of estimating the bias of $d$ coins and learning classes via non-adaptive statistical queries. Related to our results, \cite{dixon2023list} observe (at least implicitly) the basic connection that any $\rho$-replicable algorithm on $\ell(\rho)$ bits has a $2^{-\ell(\rho)}$-heavy-hitter, a basic component of the (replicability $\to$ global stability) direction of \Cref{thm:intro-parameter-free}. They also use discretized rounding, which is the core component of the reverse direction, but focus in particular on $d$-dimensional space while we give a generic discretized rounding scheme based on access to a `weakly replicable' subroutine over any domain.

DP complexity was introduced recently in \cite{canonne2024randomness}, where the authors study the number of random bits required to perform counting queries under $(\varepsilon,\delta)$-differential privacy on a fixed empirical dataset. The authors lower bound techniques use similar methods to those used to bound certificate complexity in \cite{dixon2023list}, but they do not give any generic connection between the two.

\paragraph{Stability in Agnostic Learning:} In recent independent work, \cite{blonda2025stability} also resolve \cite{chase2024local}'s question on error-dependent global stability in agnostic PAC-learning, proving a variant of \Cref{thm:intro-agnostic} with randomness complexity $\exp(d)+\poly(\alpha^{-1})$. Beyond this, \cite{blonda2025stability} also prove impossibility for error-independent global stability when OPT is bounded above by some known constant $\gamma$, generalizing the error-independent stability impossibility result of \cite{chase2024local}. We do not consider the latter model in our work.
\subsection{Technical Overview}
In this section we overview of the main ideas in the proofs of \Cref{thm:intro-parameter-free}, \Cref{thm:intro-list-DP} and \Cref{thm:intro-list-DP-user}, and \Cref{thm:intro-agnostic} respectively. We refer the reader to \Cref{sec:stable-certificate}, \Cref{sec:DP}, and \Cref{sec:PAC-cert} for formal details.
\subsubsection{Global Stability and Certificate Complexity}\label{sec:pf-view-cert}
We start with the forward direction: given a $(\frac{1}{2}+\gamma)$-replicable algorithm $\mathcal{A}_{\text{Rep}}$ for $\mathcal{M}$ on $\ell=C_{\text{Rep}}$ random bits, we transform it into a deterministic $\eta$-globally stable algorithm for $\mathcal{M}$ for any $\eta < 2^{-\ell}$. Since the g-stable complexity of $\mathcal{M}$ is the infinum of $\log\frac{1}{\eta}$ for all achievable $\eta$-globally stable algorithms, this implies $C_{\text{Glob}} \leq C_{\text{Rep}}$.

Our first step is to use $\mathcal{A}_{\text{Rep}}$ to build a (randomized) algorithm with a nearly $2^{-\ell}$-heavy-hitter. We will then transform this algorithm into a deterministic one with nearly $2^{-\ell}$ replication probability. Toward this end, observe that by averaging for every distribution $D$ there exists a random string $r_D$ such that $\mathcal{A}_{\text{Rep}}$ has a `canonical solution' $h_D$ occurring probability at least $\frac{1}{2}+\gamma$:
\[
\Pr_{S \sim D^n}[\mathcal{A}_{\text{Rep}}(S;r_D)=h_D] \geq 1/2+\gamma.
\]
By running $\mathcal{A}_{\text{Rep}}(;r)$ roughly $\frac{\log(1/\tau)}{\gamma^2}$ times and taking the majority output, we therefore get an algorithm $\mathcal{A}_{\text{maj}}$ with a nearly $2^{-\ell}$-heavy hitter
\[
\Pr_{r\sim\{0,1\}^\ell, S \sim D^{n_2}}[\mathcal{A}_{\text{maj}}(S;r)=h_D] \geq 2^{-\ell}(1-\tau),
\]
since $r=r_D$ with probability $2^{-\ell}$, and conditioned on this event $\mathcal{A}_{\text{maj}}$ almost always outputs $h_D$ by Chernoff.

We now transform $\mathcal{A}_{\text{maj}}$ into a \textit{deterministic} algorithm with nearly $2^{-\ell}$ replication probability. This follows in two steps. First, we construct a new randomized algorithm $\mathcal{A}^{\text{HH}}_{\text{maj}}$ which runs $\mathcal{A}_{\text{maj}}$ $2^{O(\ell)}$ times (over fresh samples and randomness) and output the most common response. Standard concentration then ensures the output of $\mathcal{A}^{\text{HH}}_{\text{maj}}(S;r)$ is a nearly $2^{-\ell}$-heavy-hitter of $\mathcal{A}_{\text{maj}}$ with very high probability over $S$ and $r$. Second, we derandomize $\mathcal{A}^{\text{HH}}_{\text{maj}}(S;r)$ by taking $\mathcal{A}^{\text{HH-det}}_{\text{maj}}(S)$ to be the plurality response of $\mathcal{A}^{\text{HH}}_{\text{maj}}(S;r)$ over all random strings. A simple Markov-type argument shows that $\mathcal{A}^{\text{HH-det}}_{\text{maj}}(S)$ also almost always outputs a nearly $2^{-\ell}$-heavy-hitter of $\mathcal{A}_{\text{maj}}$ over the randomness of $S$, and since there are at most $2^{\ell}$ such heavy hitters this implies $\mathcal{A}^{\text{HH-det}}_{\text{maj}}(S)$ is (nearly) $2^{-\ell}$-globally stable as desired.

We remark that correctness is maintained throughout all the above transforms so long as our initial confidence $\beta \leq 2^{-O(\ell)}$ was sufficiently small, since then any heavy hitters of the original algorithm must have been correct in the first place, and we almost always output such a hypothesis.

To prove the reverse direction, we need to show how to amplify an $\eta$-globally stable algorithm $\mathcal{A}_{\text{global}}$ to a $\rho$-replicable one using $\log \frac{1}{\eta}+\log\frac{1}{\rho}$ random bits. The key is again to look at the set of heavy-hitters of $\mathcal{A}_{\text{global}}$. For $y \in \mathcal{Y}$, let $p(y)$ denote the probability $\mathcal{A}_{\text{global}}$ outputs $y$. Running $\mathcal{A}_{\text{global}}$ sufficiently many times, we get empirical estimates $\hat{p}(y)$ such that $|\hat{p}(y)-p(y)| < \frac{\gamma}{3}$ with very high probability for some small $\gamma \ll \poly(\rho\eta)$.

Fix $T \approx \frac{1}{\eta\rho}$ and consider the set of thresholds $\{\eta - \gamma, \eta - 2\gamma,\ldots,\eta-T\gamma\}$. Our final algorithm simply selects one of the $T$ thresholds $t$ at random, and outputs the first $y$ (with respect to some fixed order on $\mathcal{Y}$) satisfying $\hat{p}(y) \geq t$. 

We now argue the above procedure is $\rho$-replicable. To see this, observe that because $\mathcal{A}_{\text{global}}$ only has $\frac{1}{\eta}$ heavy hitters of weight more than $\eta-T\gamma$ (and moreover one of weight at least $\eta$), it must be the case that at most $\frac{1}{\eta}-1$ of these thresholds are within $\gamma/3$ of some $p(y)$ for any $y \in Y$. On the other hand, conditioned on the fact $|\hat{p}(y)-p(y)|<\gamma/3$, choosing any of the $T-\frac{1}{\eta}+1$ thresholds $t$ without such a nearby heavy hitter results in a completely replicable output, since the list of $y$ with $\hat{p}(y)>t$ is always the same. The probability of selecting such a threshold is $\frac{T-\frac{1}{\eta}+1}{T} > 1- \rho$, so taking the failure probability of our empirical estimates sufficiently small gives a $\rho$-replicable algorithm as desired. See \Cref{fig:glob-to-cert} for an example diagram of this procedure.

\begin{figure}
    \centering
    \hspace{-.3cm}\includegraphics[scale=.2]{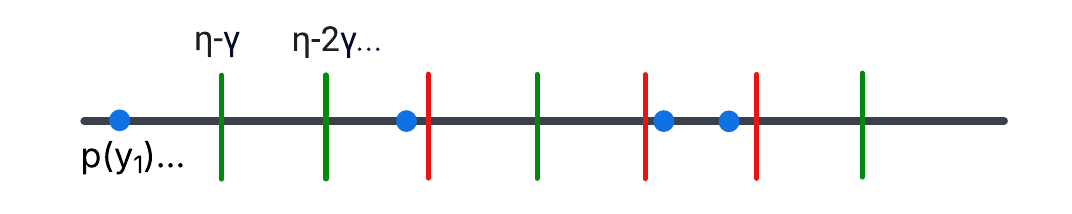}
    \caption{Thresholding procedure for $C_{\text{Glob}}=2$ and $T=7$. Blue dots denote the $4$ heavy hitters, one of which $p(y_1)$ is known to be far from any threshold. This leaves $4$ (green) thresholds with no nearby heavy-hitters out of $7$, so $\rho \approx 4/7 > \frac{1}{2}$, and $C_{\text{Rep}} \leq 3$.}
    \label{fig:glob-to-cert}
\end{figure}

Similar to before, correctness is also maintained as long as our original algorithm is $\beta \leq O(\eta)$ confident, as then every heavy hitter must be correct and the algorithm almost always outputs a heavy hitter of $\mathcal{A}_{\text{global}}$.

\subsubsection{Global Stability and Differential Privacy}\label{sec:pf-view-DP}
We now overview the main ideas behind \Cref{thm:intro-list-DP} and \Cref{thm:intro-list-DP-user} in more detail. We start with the forward direction: given an $\eta$-globally stable algorithm, we'd like to construct an $(\varepsilon,\delta)$-DP algorithm using roughly $\log(1/\eta)+\log(1/\varepsilon)+\log(1/\delta)$ random bits. The key to doing this is to first construct an $(\varepsilon,\delta/2)$-DP-algorithm $\mathcal{A}_{DP}$ satisfying the following properties
\begin{enumerate}
    \item Bounded Support: 
    \[
    \forall S: |\text{Supp}(\mathcal{A}_{DP}(S))| \leq \tilde{O}\left(\frac{1}{\eta\varepsilon}\right),
    \]
    \item Strong Correctness: 
    \[
    \Pr_{S}[\text{Supp}(\mathcal{A}_{DP}(S)) \subset G_D] \geq 1-\beta.
    \]
\end{enumerate}
In other words, the algorithm should \textit{always} have small support, and that support should almost always be \textit{completely} correct. The only issue is that $\mathcal{A}_{DP}$ might use too many random bits. We address this via an elegant observation of \cite{canonne2024randomness}: any distribution of support $T$ can be $\delta/2$-approximated (i.e.\ we can produce another distribution within $\delta/2$ in total variation distance) using only $\log(T)+\log(2/\delta)$ random bits and without adding any new elements to the support. Given the above conditions, this results in a $\beta$-confident $(\varepsilon,\delta)$-DP algorithm as desired.

Building on prior work, we construct our base DP algorithm using the `DP Selection' primitive of \cite{korolova2009releasing,bun2016simultaneous, BunDRS18}, which, roughly, given a dataset $M$, provides an $(\varepsilon,\delta)$-DP procedure to output some $y \in \mathcal{Y}$ that appears no fewer than $\frac{\log(1/\delta)}{\varepsilon}$ times less than the most frequent element in $M$. As discussed in the previous sections, given $\mathcal{A}_{\text{global}}$ we can build such a dataset $M$ by 1) building an algorithm that outputs one of $\mathcal{A}_{\text{global}}$'s heavy hitters with high probability 2) running this procedure $O(\frac{\log\frac{1}{\delta}}{\eta\varepsilon})$ times. To ensure bounded support of the algorithm over \textit{all} possible inputs, we additionally add $\log(1/\delta)/\varepsilon $ copies of some fixed `dummy hypothesis' $y \in \mathcal{Y}$ to the dataset, so DP selection always outputs something in $M$ by design. While this dummy hypothesis $y$ may not be correct, strong correctness still holds because the only way $y$ ends up in the support is if many non-heavy hitters lie in $M$ as well, a low probability event as long as the initial heavy-hitter subroutine succeeds with high probability. Finally, we remark the entire procedure is user-level private since swapping out a full sample to the heavy-hitter estimation procedure (one user) only changes one element in the database $M$.

In the reverse direction, we are given a $(T,\varepsilon,\delta)$-user-level DP algorithm $\mathcal{A}_{DP}$ on $\ell$ random bits which, by \cite{ghazi2024user}, we may also assume is $(.5,.5,.5)$-perfectly generalizing. We refer the reader to \Cref{def:perfect} for the exact definition, and for now note this implies the existence of a sample $S$ such that for all $\mathcal{O} \subset \mathcal{Y}$ we have $\Pr_r[\mathcal{A}_{DP}(S;\cdot) \in \mathcal{O}] \leq e^{1/2}\Pr_{S',r}[\mathcal{A}_{DP}(\cdot,\cdot) \in \mathcal{O}]+1/2$. Since $\mathcal{A}_{DP}(S)$ has support size at most $2^{\ell}$, setting $\mathcal{O}=\text{Supp}(\mathcal{A}_{DP}(S))$ (and therefore the LHS to $1$) implies $\Pr_{S',r}[A(\cdot,\cdot) \in \text{Supp}(\mathcal{A}_{DP}(S))] \geq \Omega(1)$, or equivalently, that $\mathcal{A}_{DP}$ has a $O(2^{-\ell})$-heavy-hitter. We may then use the same procedure as in \Cref{sec:pf-view-cert} to transform this into a deterministic $O(2^{-\ell})$-globally stable algorithm as desired.

\subsubsection{The Stability of Agnostic Learning}
Our agnostic globally stable learner is based on the agnostic-to-realizable reduction framework of \cite{hopkins2021realizable}, in particular a variant for replicable algorithms in \cite{bun2023stability}. In other words, starting with an $\eta$-globally stable learner $\mathcal{A}_{\text{PAC}}$ for the `realizable setting' (where it is assumed the underlying distribution is labeled by a hypothesis in $H$), we will build a globally stable agnostic learner $\mathcal{A}_{\text{Agn}}$ for arbitrary distributions. Our starting point is therefore the realizable setting, where we rely on the following result of \cite{ghazi2021sample}: every class $(X,H)$ with finite Littlestone dimension $d$ has an $\eta=2^{\poly(d)}\alpha^{-O(1)}$-globally stable learner on $n(\alpha,\beta)=\poly(d,\alpha,\log(1/\beta))$ samples in the realizable setting.

The core idea behind the reduction is simple: we draw a large unlabeled sample $S_U$ of size $n(\alpha,\beta)$ and run $\mathcal{A}_{\text{PAC}}$ across all possible labelings of $S_U$ in the class, then output a random hypothesis in the resulting set of outputs with low empirical error (tested across a fresh batch of labeled samples from the underlying agnostic distribution). This is an agnostic learner because the first stage always runs $\mathcal{A}_{\text{PAC}}$ over the labeled sample $(S_U,h_{OPT}(S_U))$ for some optimal hypothesis $h_{OPT}$ by design, and therefore with high probability has an output that is close to the best hypothesis in $H$. Moreover, because $\mathcal{A}_{\text{PAC}}$ is globally stable, there is even some \textit{fixed} good hypothesis $h_D$ close to $h_{OPT}$ that appears in the set with probability at least $\eta$, and therefore is given as the final output with probability $\eta/|S_U|^{O(VC(H))}$, where $VC(H)$ is the VC dimension and the divisor $|S_U|^{O(VC(H))}$ is an upper bound on the number of possible labelings of $S_U$ by $H$ (see \Cref{lemma:VC}).

This procedure almost works, but it has a core issue: to achieve confidence $\beta$, we have to set $|S_U|$ scaling with $\beta$, but then the global stability depends on $\beta$ as well which is not allowed. We fix this by running $\mathcal{A}_{\text{PAC}}$ only with confidence roughly $\eta$, ensuring at least its heavy hitter is an $\alpha$-optimal hypothesis. Unfortunately, this means we are no longer guaranteed the resulting set of outputs will have an $\alpha$-optimal hypothesis with high probability, breaking the argument. We fix this by running the procedure over $T=\frac{\log(1/\beta)}{\eta}$ independent unlabeled samples, which ensures the nearly-optimal heavy hitter $h_D$ occurs in the set with probability at least $1-\beta$. Naively we are then back at the same problem, since the output set has size $Tn(\alpha,\eta)^{O(VC)}$, but the key is to realize that $h_D$ not only occurs with high probability, it occurs \textit{at least $\Omega(\eta T)$ times} with high probability. This means we can prune all hypotheses from the resulting list that don't occur $\Omega(\eta T)$ times, cutting the total number down to at most $\frac{1}{\eta}n(\alpha,\eta)^{O(VC(H))}$. This results in an agnostic learner with a $\frac{1}{\eta}n(\alpha,\eta)^{O(VC(H))}$-heavy-hitter, which can be converted into a globally-stable learner by our prior procedure.

As written, the above has exponential sample complexity in the Littlestone dimension -- this can be fixed in the heavy-hitter setting by exploiting a stronger guarantee of \cite{ghazi2021sample}'s algorithm which in reality outputs a list of $\exp(\poly(d))\alpha^{-O(1)}$ hypotheses such that some $h_D$ appears in the list with probability at least $\Omega(\frac{1}{d})$. We can then perform the same procedure above but set $T=\poly(d,\beta)$ to get the same result with improved sample complexity.

\subsection*{Roadmap}
In \Cref{sec:prelims}, we give formal preliminaries on our setting and notions of stability. In \Cref{sec:global}, we prove our variant of global stability is equivalent to the classical heavy-hitter notion. In \Cref{sec:stable-certificate}, we prove the equivalence of g-stable and certificate complexity. In \Cref{sec:DP}, we give quantitative transforms between stable and DP complexity. In \Cref{sec:PAC-cert}, we bound the stable and certificate complexity of agnostic learning.

\section{Preliminaries}\label{sec:prelims}
We now cover our notions of stability and general setting in more formality. Most of our results are proved for the following standard and widely applicable notion of a statistical task (see e.g.\ \cite{bun2023stability}):
\begin{definition}[Statistical tasks]
    A statistical task $\mathcal{M}$ consists of: 
    \begin{enumerate}
        \item A data domain $\mathcal{X}$
        \item An output domain $\mathcal{Y}$
        \item For every distribution $D$ over $\mathcal{X}$, a subset $G_D \subset \mathcal{Y}$ of `accepted solutions'
    \end{enumerate}
    An \textbf{algorithm} solving $\mathcal{M}$ is a (possibly randomized) mapping $\mathcal{A}: \mathcal{X}^* \to \mathcal{Y}$ such that for all $\beta>0$ (the `\textbf{confidence}'), there is a function $n=n(\beta)$ (the `\textbf{sample complexity}') such that
    \[
    \forall D: \Pr_{S \sim D^n}[\mathcal{A}(S)\in G_D] \geq 1-\beta.
    \]
\end{definition}
We remark that in some statistical tasks (e.g.\ realizable PAC Learning), one only considers a restricted family of distributions $D$ over the data domain. All our results extend immediately to this generalized setting---we have chosen to focus on the distribution-free definition only for simplicity of exposition.
\subsection{The Stability Zoo}
We now overview the main notions of stability studied in this work. We remark that our definitions differ slightly from some prior works studying stability in \textit{learning}, which typically comes with an additional accuracy parameter. We discuss this distinction further in \Cref{sec:PAC}.

We start with the core notion of global stability, which measures the best replication probability achieved by any determinstic algorithm solving the task.
\begin{definition}[Global stability]\label{def:global-stability}
    Let $\mathcal{M}$ be a statistical task. An algorithm $\mathcal{A}$ solving $\mathcal{M}$ is said to be $\eta$-globally stable if for large enough $n$ and every distribution $D$
    \[
    \Pr_{S,S' \sim D^n}[\mathcal{A}(S) = \mathcal{A}(S')] \geq \eta.
    \]
    The \textit{g-stable complexity} of $\mathcal{M}$, denoted $C_{\text{Glob}}(\mathcal{M})$, is the infinum of $\log( \frac{1}{\eta})$ taken over all deterministic $\eta$-globally stable algorithms solving $\mathcal{M}$. If the set of such algorithms is empty, we write $C_{\text{Glob}}=\infty$.
\end{definition}
As mentioned previously, this differs from the `traditional' notion of global stability in \cite{bun2020equivalence}, which only requires $\mathcal{A}$ to have an $\eta$-heavy-hitter (and may use unshared internal randomness), a seemingly weaker requirement. We show these notions are fully equivalent in \Cref{sec:global}.

Global stability can never achieve replication probability beyond $1/2$ for any non-trivial statistical task \cite{chase2023replicability}. Motivated by this fact, \cite{ImpLPS22} introduced \textit{replicability}, allowing the use of shared internal randomness to boost replication success beyond the $\frac{1}{2}$ barrier.
\begin{definition}[Replicability]
     Let $\mathcal{M}$ be a statistical task. An algorithm $\mathcal{A}$ solving $\mathcal{M}$ is said to be $\rho$-replicable if for all distributions $D$:
    \[
    \Pr_{r,S,S' \sim D^n}[\mathcal{A}(S;r) = \mathcal{A}(S';r)] \geq 1-\rho.
    \]
\end{definition}
In \cite{dixon2023list}, the authors introduce \textit{certificate complexity} measuring the least number of shared random bits required to achieve $\rho$-replicability. Since we are interested in comparing certificate complexity to global stability, we will instead work mostly with a natural `parameter-free' version of this definition. In particular, we will study the number of random strings required to achieve replicability strictly better than $\frac{1}{2}$. This is the natural threshold at which many random strings are forced to have a unique `canonical' hypothesis, and can be easily amplified to arbitrary $\rho$.
\begin{definition}[Certificate Complexity]
        The \textbf{certificate complexity} of a statistical task $\mathcal{M}$, denoted $C_{\text{Rep}}(\mathcal{M})$, is the smallest number of shared random bits over which there exists a $>\frac{1}{2}$-replicable algorithm solving $\mathcal{M}$, i.e.\ some $\mathcal{A}$ satisfying:
        \[
        \Pr_{r \sim \{0,1\}^{C_{\text{Rep}}},S,S' \sim D^n}[\mathcal{A}(S;r)=\mathcal{A}(S';r)] > 1/2 
        \]
        If no such algorithm exists for any finite $C_{\text{Rep}}$, we write $C_{\text{Rep}}=\infty$.
\end{definition}
We remark that once one has achieved $> \frac{1}{2}$ replicability, \Cref{thm:intro-parameter-free} implies an amplification procedure to achieve any $\rho$-replicability parameter using an additional $\log \frac{1}{\rho}$ random bits. 

For the rest of the work, we usually omit $\mathcal{M}$ and just write $C_{\text{Glob}}$ and $C_{\text{Rep}}$ when clear from context.

\subsection{Differential Privacy}\label{sec:prelims-DP}
Global stability \cite{bun2020equivalence} and replicability \cite{GhaziKM21} (known as `pseudo-global stability' in this context) were both introduced as a tool to study the widely influential notion of \textit{differential privacy}, a powerful algorithmic guarantee promising an algorithm has similar output distributions over neighboring samples.
\begin{definition}[Differential Privacy (\cite{dwork2006calibrating,dwork2006our})]
    Let $\mathcal{M}$ be a statistical task, and $\varepsilon,\delta>0$. Two samples $S,S' \in (\mathcal{X}^n)$ are said to be neighboring if they differ in exactly one coordinate. An algorithm $\mathcal{A}$ solving $\mathcal{M}$ is said to be $(\varepsilon,\delta)$-differentially private if for any neighboring samples $S,S'$ and measurable events $\mathcal{O} \subset \mathcal{Y}$:
    \[
    \Pr[\mathcal{A}(S) \in \mathcal{O}] \leq e^{\varepsilon}\Pr[\mathcal{A}(S') \in \mathcal{O}]+\delta.
    \]
\end{definition}
We will sometimes use the notation $\mathcal{A}(S) \overset{(\varepsilon,\delta)}{=} \mathcal{A}(S')$ to denote the above closeness guarantee.

In early works \cite{bun2020equivalence,ghazi2021sample,GhaziKM21,ghazi2024user,bun2023stability}, global stability and its variants played a key role in bounding the sample complexity of differentially private PAC learning. Quite recently, \cite{canonne2024randomness} expanded this connection by using techniques developed for certificate complexity \cite{dixon2023list,woude2023geometry} to bound the number of random bits needed to achieve differential privacy, a parameter we'll call \textit{DP complexity}.
\begin{definition}[Parametrized DP Complexity]
    The $(n,\beta,\varepsilon,\delta)$-DP Complexity of a statistical task $\mathcal{M}$ is the smallest integer $C_{DP}(n,\beta,\varepsilon,\delta)$ such that there exists an $(\varepsilon,\delta)$-DP algorithm on $n$ samples solving $\mathcal{M}$ with confidence $\beta$ using only $C_{DP}(n,\beta,\varepsilon,\delta)$ random bits. If no such algorithm exists for any finite $C_{DP}(n,\beta,\varepsilon,\delta)$, we write $C_{DP}(n,\beta,\varepsilon,\delta)=\infty$.
\end{definition}
Unlike certificate complexity, it is not clear there is a natural `parameter-free' variant of DP complexity, especially since the concept frequently is used at many different scales of $(\varepsilon,\delta)$ in theory and practice, themselves dependent on $n$ and $\beta$. As such, we will focus mostly on parameter-dependent translations in this setting, and define the following `parametrized' variants of g-stable and certificate complexity.
\begin{definition}[Parametrized g-Stable Complexity]
    The $(n,\beta)$-g-stable complexity of a task $\mathcal{M}$, denoted $C_{\text{Glob}}(n,\beta)$, is $\log \frac{1}{\eta}$ where $\eta$ is the largest global stability achieved by any $\beta$-confident algorithm on $n$ samples solving $\mathcal{M}$. If no such algorithm exists, we write $C_{\text{Glob}}(n,\beta)=\infty$.
\end{definition}
\begin{definition}[Parametrized Certificate Complexity]
    The $(n,\beta)$-certificate complexity of a task $\mathcal{M}$, denoted $C_{\text{Rep}}(n,\beta)$, is the smallest number of random bits such that there is a $\beta$-confident, better than $\frac{1}{2}$-replicable algorithm on $n$ samples solving $\mathcal{M}$. If no such algorithm exists, we write $C_{\text{Rep}}(n,\beta)=\infty$.
\end{definition}
\subsection{User Level Privacy}
In \cite{GhaziKM21,bun2023stability}, the authors take advantage of the strong stability guarantees of replicability to build `\textit{user-level}' private algorithms \cite{levy2021learning}, a stronger notion of privacy in which each user is thought of as providing many data points and the goal is to protect against an adversary swapping out a single user's \textit{entire} dataset. Formally, the input dataset $S$ is actually split into subsets $S_1,\ldots,S_T$ coming from $T$-separate users, and we bound the algorithms deviation upon switching out one such $S_i$:

% As before, to compare DP complexity to other notions we need a parameter-free variant. Unlike replicability, it is not immediately clear what the right `parameter threshold' for $(\varepsilon,\delta)$ should be. We will choose our notion in such a way that it is quantitatively `equivalent' to the parametrized variant in the sense that it is typically easy to move between the two with only constant loss. This requires introducing a slightly finer-grained notion of differential privacy called \textit{user-level privacy}, which models when the data-set $S$ is actually split into subsets $S_1,\ldots,S_T$ coming from $T$-separate users. User-level DP bounds the algorithm's deviation upon swapping out an entire user data-set rather than a single example.
\begin{definition}[User-Level Differential Privacy]
    Let $\mathcal{M}$ be a statistical task, $n,T \in \mathbb{N}$ such that $T$ divides $n$, and call samples $(S_1,\ldots,S_{T}), (S'_1,\ldots,S'_{T}) \in (\mathcal{X}^{n/T})^T$ neighboring if $S'_j=S_j$ for all but one $j \in [T]$. We call an algorithm $\mathcal{A}$ solving $\mathcal{M}$ $T$-user $(\varepsilon,\delta)$-DP if $\mathcal{A}(S) \overset{(\varepsilon,\delta)}{=} \mathcal{A}(S')$ for all such neighboring data-sets $S,S'$ 
    % and measurable events $\mathcal{O} \subset \mathcal{Y}$:
    % \[
    % \Pr[\mathcal{A}(S) \in \mathcal{O}] \leq e^{\varepsilon}\Pr[\mathcal{A}(S') \in \mathcal{O}]+\delta
    % \]
\end{definition}
User-level DP allows us to tighten the connection between DP complexity and certificate complexity, giving a tight transformation between the two up to constants for reasonable regimes of $\varepsilon$ and $\delta$.
\begin{definition}[User-Level DP Complexity]
        The $(T,\varepsilon,\delta)$-user-level-DP Complexity of a statistical task $\mathcal{M}$ is the smallest integer $C_{DP}(T,\varepsilon,\delta)$ such that there exists a $T$-user $(\varepsilon,\delta)$-DP algorithm solving $\mathcal{M}$ using only $C_{DP}(T,\varepsilon,\delta)$ random bits. If no such algorithm exists for any finite $C_{DP}(T,\varepsilon,\delta)$, we write $C_{DP}(T,\varepsilon,\delta)=\infty$.
\end{definition}

\subsection{PAC Learning}\label{sec:PAC}
Having established close quantitative connections between globally stable, certificate, and DP complexity, our second focus is to concretely determine the randomness complexity of a fundamental statistical task: binary classification. We will focus on the standard PAC framework, in which we have a data domain $X$ and a \textit{hypothesis class} $H=\{h:X \to \{0,1\}\}$, a collection of possible labelings of $\mathcal{X}$. Let $D$ be a distribution over labeled samples $(X \times \{0,1\})$ and $h: X \to \{0,1\}$ a potential labeling. The error of $h$ under $D$ is
\[
err_D(h) = \Pr_{(x,y) \sim D}[h(x) \neq y].
\]
Similarly the empirical error of $h$ on a labeled sample $S$ is
\[
err_S(h) = \Pr_{(x,y) \in S}[h(x) \neq y].
\]
We write the best possible error achieved over $H$ as
\[
OPT_D \coloneqq \inf_{h \in H} err_D(h).
\]
We call a hypothesis $h$ $\alpha$-optimal if it achieves error at most $OPT+\alpha$. An algorithm is said to (agnostic) PAC learn the class $(\mathcal{X},H)$ if for every $\alpha,\beta>0$, on sufficiently many samples $\mathcal{A}$ outputs an $\alpha$-optimal hypothesis with probability at least $1-\beta$.
\begin{definition}[(Agnostic) PAC Learning]
    We say a class $(X,H)$ is PAC-Learnable if $\forall \alpha,\beta>0$ there exists $n=n(\alpha,\beta)$ and a (possibly randomized) algorithm $\mathcal{A}: \mathcal{X}^n \to P(\mathcal{X})$ satisfying
    \[
    \Pr_{S \sim D^n}[err_D(\mathcal{A}(S)) >OPT_D + \alpha] < \beta.
    \]
\end{definition}
Note that PAC-Learning, as formalized above, corresponds to an \textit{infinite sequence} of statistical tasks $\{\mathcal{M}^{(X,H)}(\alpha)\}_{\alpha>0}$ parametrized by the accuracy $\alpha$, namely where the data domain is $\mathcal{X} = X \times \{0,1\}$, the output space is all possible labelings $\mathcal{Y} = \{h: X \to \{0,1\}\}$, and the set of accepted solutions is $G_D \coloneqq \{h \in H: err_D(H) \leq OPT + \alpha\}$. With this in mind, we may therefore define the `error-dependent' stable/certificate/DP-complexity of a class $(X,H)$ as 
\[
C_{\text{Glob}}(\alpha) = C_{\text{Glob}}(\mathcal{M}(\alpha)),
\]
and likewise for certificate and DP complexity.

We emphasize this differs from prior works \cite{bun2020equivalence,chase2024local} that focus on g-stable complexity bounds that are \textit{independent} of the excess error $\alpha$. In \cite{bun2020equivalence}, the authors prove that any class with finite Littlestone dimension $d$ (we refer the reader to \cite{bun2020equivalence} for the exact combinatorial definition, which is not relevant for our work) has a universal upper bound on the g-stable complexity of $2^{O(d)}$ under the assumption that $\text{OPT}=0$. \cite{chase2024local} showed this result cannot extend in the full agnostic setting, and indeed that \textit{no} error-independent bound is possible for any infinite class. They ask whether it is possible to give an error-dependent bound, a question we resolve in the positive in \Cref{thm:intro-agnostic}.

Without stability constraints, the sample complexity of PAC learning is tightly controlled by a combinatorial parameter called the VC-dimension. Again the exact definition is not relevant to our work, but to answer \cite{chase2024local}'s question we will require the following classical lemma bounding the number of labelings of any subset of data as a function of the VC dimension.
\begin{lemma}[The Sauer-Shelah-Perles Lemma {\cite{vapnik1974theory,sauer1972density,shelah1972combinatorial}}]\label{lemma:VC}
    Let $(X,H)$ be a hypothesis class. For any $n \in \mathbb{N}$ and $S \subset X$ of size $n$, the number of labelings of $S$ by hypotheses in $H$ is at most $n^{O(VC(H))}$.
\end{lemma}
We note that the VC dimension is always at most the Littlestone dimension.

\section{Global Stability: Replication vs Heavy-Hitters}\label{sec:global}
In this section we prove that \Cref{def:global-stability} (replication global stability) is equivalent to the traditional randomized `heavy-hitters' variant studied in \cite{bun2020equivalence,chase2023replicability}. Given a distribution $D$ over some domain $X$, recall an element $x \in X$ is called an $\eta$-heavy-hitter if the measure of $x$ in $D$ is at least $\eta$. Traditional global stability promises the existence of a heavy hitter over the output distribution of the algorithm.
\begin{definition}[Heavy-Hitter Global stability {\cite{bun2020equivalence}}]
    Let $\mathcal{M}$ be a statistical task. An algorithm $\mathcal{A}$ solving $\mathcal{M}$ is said to be $\eta$-heavy-hitter globally stable if for large enough $n$ and every distribution $D$ there exists $h_D$ such that:
    \[
    \Pr_{S,S' \sim D^n}[\mathcal{A}(S)=h_D] \geq \eta.
    \]
    The \textit{HH-stable complexity} of $\mathcal{M}$, denoted $C_{\text{Glob-HH}}$ is the infinum of $\log( \frac{1}{\eta})$ taken over all (possibly randomized) $\eta$-heavy-hitter globally stable algorithms solving $\mathcal{M}$.
\end{definition}
In \cite{chase2023replicability}, the authors observe that HH-global stability and global stability as in \Cref{def:global-stability} differ by at most a quadratic factor. We prove they are actually the same quantity, a fact we will use later both to move between randomized and deterministic algorithms, and to move from heavy-hitter bounds to global stability.
\begin{theorem}[Global Stability vs HH-Global Stability]\label{thm:old-to-new-global}
    For any statistical task $\mathcal{M}$: $C_{\text{Glob-HH}} = C_{\text{Glob}}$
\end{theorem}
\begin{proof}
    We remark that any $\eta$-globally stable algorithm is automatically $\eta$-heavy-hitter globally stable, so $C_{\text{Glob-HH}} \leq C_{\text{Glob}}$ is immediate. Thus it remains to prove $C_{\text{Glob}} \leq C_{\text{Glob-HH}}$.

    Write $\eta =  2^{-C_{\text{Glob-HH}}}$. For any $\gamma,\beta>0$, we are promised a $\beta$-correct, possibly randomized algorithm $\mathcal{A}$ such that for any distribution $D$ there exists a hypothesis $h_D$ satisfying
    \[
    \Pr_{r,S \sim D^n}[\mathcal{A}(S) = h_D] \geq \eta - \gamma
    \]
    Naively, $\mathcal{A}$ only has replication probability approaching $\eta^2$. To improve this, the idea is to build a new algorithm that outputs one of $\mathcal{A}$'s $(\eta-2\gamma)$-heavy-hitters with probability at least $1-\gamma$. By standard concentration bounds (see e.g.\ \cite{kontorovich2024distribution}) this can be done simply by running $\mathcal{A}$ on $O(\frac{\log(1/\gamma)}{\eta^2})$ fresh samples and outputting the most common hypothesis in the list (breaking ties arbitrarily). Toward this end, denote the list of $(\eta-2\gamma)$-heavy-hitters by $L_D$, and observe that for small enough $\gamma \leq O(\eta)$, $|L_D| \leq \frac{1}{\eta}$.

    Taking $\beta \leq O(\eta)$ sufficiently small, we can also ensure all elements of $L_D$ are correct solutions, i.e.\ that $L_D \subset G_D$. We have therefore arrived at an algorithm $\mathcal{A}_{\text{List}}$ with the following property: for any distribution $D$, there exists a list $L_D \subset G_D$ of size at most $\frac{1}{\eta}$ such that 
    \[
    \Pr_{r,S \sim D^n}[\mathcal{A}_{\text{List}}(S;r) \in L_D] \geq 1-\gamma.
    \]
    We note algorithms satisfying this guarantee are called `$|L_D|$-list-replicable' \cite{chase2023replicability,dixon2023list}.
    
    We are now almost done since conditioned on landing in $L_D$ in both runs, the replication probability of $\mathcal{A}_{\text{List}}$ is at least $\frac{1}{|L_D|} \geq \eta$, and the probability of being outside the list can be taken to arbitrarily small. The only issue is that $\mathcal{A}_{\text{List}}$ is randomized. Thus it is sufficient to argue we can de-randomize $\mathcal{A}_{\text{List}}$ at the cost of increasing the list-failure probability $\gamma$ to some corresponding $\gamma'$ also arbitrarily small.
    
    We do this simply by passing to the most likely output of $\mathcal{A}_{\text{List}}$. Namely define $\mathcal{A}_{\text{det}}$ by
    \[
    \mathcal{A}_{det}(S) \coloneqq \underset{y \in \mathcal{Y}}{\text{argmax}}\{\Pr[\mathcal{A}_{\text{List}}(S)=y]\},
    \]
    and note that $\mathcal{A}_{det}(S)$ can be easily computed by simply running $\mathcal{A}_{\text{List}}(S;r)$ for all choices of internal randomness $r$ and outputting the most common result (breaking ties arbitrarily). We claim:
    \[
    \Pr_S[\mathcal{A}_{det}(S) \in L_D] = \Pr_S\left[\underset{y \in Y}{\text{argmax}}\{\Pr[\mathcal{A}_{\text{List}}(S)=y]\} \in L_D\right] \geq 1-\frac{2\gamma}{\eta}.
    \]
    This follows from the observation that for any sample $S$ whose maximal output is not in $L_D$, the probability of outputting an element outside of $L$ is then at least $\frac{\eta}{2}$ (either the maximum probability, by assumption achieved by an element outside of $L_D$ is at least $\frac{\eta}{2}$ in which case we are done, or is at most $\frac{\eta}{2}$ in which case the list elements make up only $1/2$ the mass by assumption). Finally since we may take $\gamma$ arbitrarily small with respect to $\eta$, taking a large enough sample we may make $\mathcal{A}_{\text{det}}$ $\eta-\gamma'$ globally stable for any $\gamma'>0$, implying $C_{\text{Glob}} \geq \log \frac{1}{\eta} = C_{\text{Glob-HH}}$ as desired.
\end{proof}
It will also be useful to have a parametrized version of the above transform:
\begin{corollary}\label{cor:parametrized-old-to-new}
    For any statistical task $\mathcal{M}$ and $\eta,\beta>0$, given an $\eta$-HH globally stable, $\beta$-confident algorithm for $\mathcal{M}$ on $n$ samples, the transform in \Cref{thm:old-to-new-global} gives an $O(\eta)$-globally stable $\beta'$-confident algorithm for $\mathcal{M}$ on $n'$ samples, where $\beta' \leq O(\beta/\eta)$ and $n' \leq n\cdot \tilde{O}\left(\frac{\log \frac{1}{\beta}}{\eta^2}\right)$.
\end{corollary}
\begin{proof}
    Note that the statement is trivial if $\beta > \eta$, and take $\gamma,\gamma' \leq O(\beta\eta)$ in the proof of \Cref{thm:old-to-new-global}. For such parameters $\mathcal{A}_{\text{det}}$ runs $\mathcal{A}$ at most $\tilde{O}\left(\frac{\log \frac{1}{\beta}}{\eta^2}\right)$ times, and has confidence $\beta' \leq 2\gamma/\eta \leq \beta/\eta$ (since it outputs an element in $L_D$ with at least this probability, which are correct by the assumption on $\beta$).
\end{proof}
\section{Global Stability and Certificate Complexity}\label{sec:stable-certificate}
In this section, we prove \Cref{thm:intro-parameter-free}, the (near) equivalence of g-stable and certificate complexity. We repeat the main result of the section here for convenience.
% \begin{theorem}[List Replicability vs Certificate Complexity]\label{thm:equiv}
%     For any statistical task $\mathcal{M}$:, 
%     \[
%     \log_2(L_{\mathcal{M}}^{-1}) \leq \ell_\mathcal{M} \leq \log(2L_\mathcal{M}^{-1}-1),
%     \]
%     where $\et\mathcal{A}_{\mathcal{M}}$ is the largest global stability parameter achievable for $\mathcal{M}$.
% \end{theorem}
\begin{theorem}[Global Stability vs Certificate Complexity]\label{thm:equiv}
    For any statistical task $\mathcal{M}$:
    \[
    C_{\text{Glob}} \leq C_{\text{Rep}} \leq C_{\text{Glob}}+1.
    \]
    Moreover, the number of random bits required to achieve $\rho$-replicability is at most $\lceil C_{\text{Glob}}+\log(1/\rho)\rceil $
\end{theorem}
We start with the easy direction: certificate complexity to global stability. It is easy to see that any better-than-half replicable algorithm on $\ell$-bits is automatically $2^{-\ell-1}$-globally stable (albeit randomized), since there must exist some $r \in \{0,1\}^\ell$ which has $>\frac{1}{2}$ replication probability. This can be amplified to probability $2^{-\ell}-\eta$ for any $\eta>0$ simply by running the algorithm many times and taking majority.

\begin{proof}[Proof of \Cref{thm:equiv} (Lower Bound)]
Recall by \Cref{thm:old-to-new-global}, it is enough to prove for any $\tau>0$ the existence of a \textit{randomized} algorithm $\mathcal{A}$ solving $\mathcal{M}$ such that for any distribution $D$
\[
\Pr_{S,S' \sim D, r,r'}[\mathcal{A}(S;r)=\mathcal{A}(S';r')] \geq 2^{-C_{\text{Rep}}} - \tau
\]
where $r$ and $r'$ are independent random strings.
% (namely such an algorithm is $(2^{-C_{\text{Rep}}} - \tau)$-heavy-hitter globally stable).

Now by assumption there exists a $(\frac{1}{2}+\gamma)$-replicable algorithm solving $\mathcal{M}$ using $C_{\text{Rep}}$ random bits. By averaging, we then have that for any distribution $D$, there exists a string $r_D \in \{0,1\}^{C_{\text{Rep}}}$ such that
\[
\Pr_{S,S' \sim D^n}[\mathcal{A}(S;r)=\mathcal{A}(S';r)] \geq \frac{1}{2}+\gamma
\]
and therefore that there exists a `canonical hypothesis' $h_D$ such that
\[
\Pr_{S,S' \sim D^n}[\mathcal{A}(S;r)=h_D] \geq \frac{1}{2}+\gamma
\]
Given this, consider the `majority amplified' algorithm $\mathcal{A}_{\text{maj-}T}$ on $nT$ samples and the same internal randomness defined as
\[
\mathcal{A}_{\text{maj-}T}(S_1,\ldots,S_T; r) = \text{plurality}\{\mathcal{A}(S_i;r)\}
\]
breaking ties arbitrarily. Taking $T=O\left(\frac{\log \frac{1}{\tau}}{\gamma^2}\right)$ large enough, Chernoff promises
\[
\Pr_{S \sim D^{nT}}[\mathcal{A}_{\text{maj-}T}(S;r) = h_D] \geq 1-\frac{\tau}{2},
\]
so the total collision probability over two runs on independent strings is at least
\begin{align*}
    \Pr[\mathcal{A}(S;r)=\mathcal{A}(S';r')] &\geq \Pr[r=r'=r_D]\cdot\Pr[\mathcal{A}(S;r_D)=\mathcal{A}(S';r_D)~|~r=r'=r_D]\\
    & \geq 2^{-C_{\text{Rep}}}(1-\tau)
\end{align*}
as desired.
\end{proof}

We now prove the reverse direction, which follows from a careful discretization of \cite{ImpLPS22}'s random thresholding method. We first give the algorithm in pseudocode.

\begin{algorithm}[H]

\KwResult{$(\rho - \tau')$-replicable, $\beta$-confident algorithm on $\lceil \frac{2^{C_{\text{Glob}}}-1}{\rho} \rceil$ random strings}
\nonl \textbf{Input:} $(2^{-C_{\text{Glob}}}-\tau)$-globally stable and $O(2^{-C_{\text{Glob}}})$-confident algorithm $\mathcal{A}$, total ordering $\phi$ of $\mathcal{Y}$.\\
\nonl \textbf{Parameters:} 
\begin{itemize}
    \item Threshold number $T=\lceil \frac{2^{C_{\text{Glob}}}-1}{\rho} \rceil$
    \item Heavy-Hitter parameter $\eta = 2^{-C_{\text{Glob}}}$
    \item Threshold offset $ \tau \ll \gamma \ll \eta$
    \item Amplification parameter $N(\gamma,\beta,\tau')$
\end{itemize}
\nonl \textbf{Algorithm:}\\
\begin{enumerate}
    \item Run $\mathcal{A}$ across $N$ fresh samples $S \sim D^n$ and let $\hat{p}_h$ denote the empirical density of each $h \in \mathcal{Y}$
    \item Select $i \in [T]$ uniformly at random, and let
    \[
    H_i \coloneqq \{h: \hat{p}_h \geq \eta-i\gamma\}
    \]
\end{enumerate}
\textbf{Return} $\underset{H_i}{\text{argmin}} \ \hat{p}_h$ (breaking ties via $\phi$) or $\bot$ if $H_i=\varnothing$
 \caption{Global Stability $\to$ Certificate Complexity}
\label{alg:global-to-cert}
 
\end{algorithm}

\begin{proof}[Proof of \Cref{thm:equiv} (Upper Bound)]
    We prove a slightly stronger result: for any $\tau'>0$, there is a $(\rho-\tau')$-replicable learner using $T=\lceil \frac{2^{C_{\text{Glob}}}-1}{\rho}\rceil$ random strings. Achieving $\rho>\frac{1}{2}+\frac{1}{2^{O(C_{\text{Glob}})}}$ can therefore be done in $C_{\text{Glob}}+1$ random bits, and general $\rho$ may be achieved using $\lceil C_{\text{Glob}}+\log(1/\rho) \rceil$ bits as desired.\footnote{Formally, we note there is some discrepancy here in sampling a random threshold for $T$ not a power of $2$ when $r$ is a bit-string. However, in the argument below increasing the number of random strings/thresholds to the nearest power of 2 only helps replicability, so this is not an actual issue.}

    To prove the stronger statement, first observe that by standard concentration inequalities (see e.g.\ \cite{kontorovich2024distribution}), for large enough $T_1 \leq O(\frac{\log(\frac{1}{\beta\tau'})}{\gamma^2})$ with probability at least $1-\beta\tau'$ all empirical estimates $\hat{p}_h$ in Step (1) of \Cref{alg:global-to-cert} satisfy 
    \begin{equation}\label{eq:empirical-est}
            |\hat{p}_h-p| < \gamma/3.
    \end{equation}
    Condition on \Cref{eq:empirical-est} occurring. Choosing $\tau \leq O(\gamma) \leq O(T^{-1})$ sufficiently small, there are at most $\frac{1}{\eta}$ hypotheses of empirical weight greater than $\eta- T\gamma$.
    % Taking $T_1$ sufficiently large, there are therefore also at most $\frac{1}{\eta}$ hypotheses with empirical measure greater than $\eta- (\frac{2}{\eta}-1)\gamma$ except with some negligible probability $\beta$. 
    Since $\mathcal{A}$ must have an $(\eta-\tau)$-heavy hitter, there also exists at least one $\hat{p} \geq \eta-\tau-\gamma/3 > \eta-2\gamma/3$, so the set of empirical heavy hitters $H_i$ (defined in \Cref{alg:global-to-cert} Step 2) is non-empty for any $i$ and \Cref{alg:global-to-cert} will not output $\bot$.

Now consider the set of $T$ thresholds $\{\eta-\gamma,\eta-2\gamma,\ldots, \eta-T\gamma\}$.
% \footnote{Note that taking any $\ell>\frac{2}{\eta}-1$ thresholds for the right setting of parameters only improves replicability.} 
By assumption, at most $\frac{1}{\eta}-1$ of these thresholds have a hypothesis with true weight within $\gamma/3$ of their value. Thus choosing a random threshold, the probability we select one with no nearby hypothesis is at least $1-\frac{\frac{1}{\eta}-1}{T}$. Conditioned on selecting such a threshold $t$ and \Cref{eq:empirical-est}, the set of hypotheses with empirical measure greater than $t$ is always the same, so the algorithm always outputs the smallest such element in the list according to $\phi$ as desired.
% Moreover, given we select such a threshold $t$, consider the set $L_t$ of hypotheses in $L_D$ with empirical measure at least $t$. First, observe $L_t$ is non-empty, since it contains at least the promised hypothesis with true measure at least $\frac{1}{L}-\gamma/3$. Second, $L_t$ is replicable by our conditioning on the accuracy of $\hat{p}_h$. 
In total, it follows the algorithm is $\left(\frac{2^{C_{\text{Glob}}}-1}{\lceil \rho^{-1}(2^{C_{\text{Glob}}}-1) \rceil}-\tau'\right)$-replicable

Finally, since the algorithm is assumed to be $O(2^{-C_{\text{Glob}}})$-confident, any true $\Omega(\eta)$-heavy hitter must also be correct. Conditioned on \Cref{eq:empirical-est} \Cref{alg:global-to-cert} always outputs such a heavy hitter, so the final algorithm is $\beta$-confident as well.
\end{proof}
We note that in the setting of achieving strictly greater than $\frac{1}{2}$ replicability, the sample overhead in \Cref{thm:equiv} can be taken as at worst $2^{O(C_\text{Glob})}$, since all parameters $\tau,\tau',\gamma$ can be set to $2^{-O(C_\text{Glob})}$ and still achieve the desired replicability. We will use this fact in proving \Cref{thm:intro-agnostic}.
% I conjecture that the above bound is tight for the coin problem:
% \begin{conjecture}[Coin Problem]
%     Estimating the bias of a single coin requires $|R|\geq 3$ for $\rho>\frac{1}{2}$
% \end{conjecture}
% The coin problem has list-size $2$, so this would separate list size and randomness. It is possible that for $\rho=\frac{1}{2}$ the notions coincide. It would be interesting to see whether this extends to the $d$-coin problem:
% \begin{conjecture}[$d$-Coin Problem]
%     Estimating the bias of $d$ coins requires $|R|\geq 2(d+1)-1$ for $\rho>\frac{1}{2}$
% \end{conjecture}
\section{Stability and Differential Privacy}\label{sec:DP}
We now prove \Cref{thm:intro-list-DP} and \Cref{thm:intro-list-DP-user} connecting global stability and the randomness complexity of differential privacy. We restate the results in more formality.
\begin{theorem}[Stability Boosting (DP)]\label{thm:DP-list}
    There exists a universal constant $c>0$ such that for any statistical task $\mathcal{M}$:
    \begin{enumerate}
        \item \textbf{(Stability to DP):} $C_{DP}(n,\beta,\varepsilon,\delta) \leq C_{\text{Glob}}(n',\beta')+\log(1/\varepsilon)+\log(1/\delta)+\log\log(1/\delta)+O(1)$
        \\
        \\
        \vspace{.2cm}
        \noindent \hspace{-1cm}for any $n \geq n'\cdot O(\frac{2^{3C_{\text{Glob}}(n',\beta')}\log\frac{1}{\beta'}\log\frac{1}{\delta}}{\varepsilon})$, $\beta \geq 2^{C_{\text{Glob}}(n',\beta')}\frac{\log \frac{1}{\delta}}{c\varepsilon}(\beta'2^{C_{\text{Glob}}(n',\beta')+1})^{O(\frac{\log\frac{1}{\delta}}{\varepsilon})}$, and
        \item \textbf{(DP to Stability):} $C_{\text{Glob}}(n,\beta) \leq C_{DP}(n',\beta',\varepsilon,\delta) + O(1)$
    \end{enumerate}
    for any $\varepsilon \leq \frac{c}{\sqrt{n'\log(n')}}$ and $\delta \leq \frac{c}{n'}$, $n \geq n' \cdot \tilde{O}\left(2^{2C_{DP}(n',\beta',\varepsilon,\delta)}\log \frac{1}{\beta'}\right)$, and $\beta \geq O(\beta'2^{C_{DP}(n',\beta',\varepsilon,\delta)})$.
\end{theorem}
\begin{theorem}[Stability Boosting (User-Level DP)]\label{thm:list-DP-user}
        There exist universal constants $c_1,c_2>0$ such that for any statistical task $\mathcal{M}$:
    \begin{enumerate}
        \item \textbf{(Stability to DP):} $C_{DP}\left(2^{C_{\text{Glob}}}\frac{c_1\log \frac{1}{\delta}}{\varepsilon},\varepsilon,\delta\right) \leq C_{\text{Glob}}+\log(1/\varepsilon)+\log(1/\delta)$
        \item \textbf{(DP to Stability):} $C_{\text{Glob}} \leq C_{DP}(T,\varepsilon,\delta) + O(1)$
    \end{enumerate}
    where the latter holds for any $(T,\varepsilon,\delta)$ satisfying $\varepsilon \leq \frac{c_2}{\sqrt{T\log(T)}}$ and $\delta \leq \frac{c_2}{T}$.
\end{theorem}
Both \Cref{thm:DP-list} and \Cref{thm:list-DP-user} follow as corollaries of the same underlying stability to DP and DP to stability transformations. We start by proving the forward direction: a randomness efficient stability-to-DP transform. The key to achieving low randomness is the following useful observation of \cite{canonne2024randomness}: any distribution with small support can be approximately sampled using few random bits.
\begin{lemma}[{\cite[Lemma 2.10 (rephrased)]{canonne2024randomness}}]\label{lemma:less-random}
    For any randomized algorithm $M: X^n \to Y$ and $\eta>0$, there exists an algorithm $M'$ using $\max_{x}\log(|\text{Supp}(M(x))|) + \log(1/\eta)$ random bits such that for every input $x$, $M'$ satisfies:
    \begin{enumerate}
        \item \textbf{(Closeness):} $d_{TV}(M(x),M'(x)) \leq \eta$
        \item \textbf{(Subset Support):} $\text{Supp}(M'(x)) \subseteq \text{Supp}(M(x))$
    \end{enumerate}
\end{lemma}
The result will now follow if we can prove a stability-to-DP transformation with two key properties
\begin{enumerate}
    \item For every $S$, the support of $\mathcal{A}(S)$ is small
    \item With high probability over $S$, $\text{Supp}(\mathcal{A}(S)) \subset G_D$,
\end{enumerate}
where we recall $G_D$ is the set of correct solutions for the task $\mathcal{M}$. The stronger correctness property is needed since applying \Cref{lemma:less-random} might otherwise ruin the correctness of our algorithm.

To build such a transformation, we will use the following DP Selection algorithm also used as the main subroutine in the standard stability-to-DP transformation of prior works \cite{bun2020equivalence,GhaziKM21,bun2023stability}.
\begin{theorem}[DP Selection \cite{korolova2009releasing,bun2016simultaneous, BunDRS18}, as stated in {\cite{bun2023stability}}]\label{thm:DP-Selection}
There exists some $c>0$ such that for every $\varepsilon,\delta>0$ and $m \in \mathbb{N}$, there is an ($\varepsilon,\delta$)-DP algorithm that on input $S \in \mathcal{\mathcal{X}}^m$, outputs with probability $1$ an element $x \in X$ that occurs in $S$ at most $\frac{c\log \delta^{-1}}{\varepsilon}$ fewer times than the mode of $S$.
\end{theorem}
We can now state and prove our bounded support, strong correctness stability-to-DP transform:
\begin{lemma}\label{lemma:list-to-DP}
    Let $\mathcal{M}$ be a statistical task, $\eta,\varepsilon,\delta>0$, and $T = O\left(\frac{\log(1/\delta)}{\eta\varepsilon}\right)$. Given an $\eta$-globally stable algorithm $\mathcal{A}$ for $\mathcal{M}$ on $n=n(\beta)$ samples, there exists a $T$-user, $\beta'$-confident, $(\varepsilon,\delta)$-DP algorithm $\mathcal{A}_{DP}$ on $n'=n(\varepsilon,\delta,\beta)$ samples satisfying:
    \begin{enumerate}
        \item \textbf{Bounded Support:} $\forall S: |\text{Supp}(\mathcal{A}_{DP}(S;))| \leq T$
        \item \textbf{Strong Correctness:} $\Pr_{S \sim D^n}[\text{Supp}(\mathcal{A}_{DP}(S;)) \subset G_D] > 1-\beta'$
    \end{enumerate}
    for $\beta' \leq T(\frac{2\beta}{\eta})^{O(\frac{\log \frac{1}{\delta}}{\varepsilon})}$ and $n' \leq n(\beta)\cdot O\left(\frac{\log(1/\beta)\log(1/\delta)}{\eta^3\varepsilon}\right)$
\end{lemma}
\begin{proof}
    By assumption we are given access to a $\beta$-confident algorithm $\mathcal{A}$ on $n=n(\beta)$ samples that has an $\eta$-heavy hitter $h^*$. In order to ensure strong correctness, we will first need to transform $\mathcal{A}$ into a so-called `list-replicable' algorithm $\mathcal{A}_{\text{List}}$, that is one such that there exists a small list $L \subset G_D$ of correct hypotheses such that for a large enough $m$:
    \[
    \Pr_{S \sim D^m}[\mathcal{A}_{\text{List}}(S) \in L] \geq 1-\beta
    \]
    Assume $\beta \leq \eta/2$ (else the Lemma statement is trivial as $\beta'>1$). Then every $\eta/2$-heavy-hitter of $\mathcal{A}$ must lie in $G_D$. We build $\mathcal{A}_{\text{List}}$ by simply running $\mathcal{A}$ on $O\left(\frac{\log(\frac{1}{\beta})}{\eta^2}\right)$ fresh samples and outputting the most common result. By standard concentration \cite{kontorovich2024distribution}, the probability that the output of this procedure is not an $\eta/2$-heavy hitter of $\mathcal{A}$ (of which there are at most $\frac{2}{\eta}$) is at most $\beta$ as desired, and the algorithm uses at most $m=O(n\frac{\log(1/\beta)}{\eta^2})$ samples.

    We will now use $\mathcal{A}_{\text{List}}$ to generate a dataset of output hypotheses on which to apply DP selection. More formally, fix some arbitrary `dummy' output $y \in \mathcal{Y}$ (we use this to ensure the end support is bounded) and consider the following procedure generating a dataset to run DP-Selection:
    \begin{enumerate}
        \item Draw $T=O\left(\frac{\log(1/\delta)}{\eta\varepsilon}\right)$ independent size-$m$ samples from $D$
        \item Add the output of $\mathcal{A}_{\text{List}}$ on every sample in $T$ to the dataset
        \item Add $\frac{c\log \frac{1}{\delta}}{\varepsilon}$ copies of $y$ for $c$ as in \Cref{thm:DP-Selection}.
    \end{enumerate}
    Finally, define $\mathcal{A}_{DP}(S)$ to be the algorithm that outputs the result of DP-Selection on the above dataset.
    
    We first analyze correctness. Note it is enough to argue that with probability at least $1-\beta'$, the support is contained in $L$, since $L \subset G_D$. By \Cref{thm:DP-Selection}, the only way the support of $\mathcal{A}_{DP}(S)$ contains a hypothesis outside $L$ is if $\mathcal{A}_{\text{List}}$ outputs $h \notin L$ in Step 2 at least $O(\log(1/\delta)/\varepsilon)$ times.\footnote{Formally there may be some other measure 0 set of outputs given the exact statement of \Cref{thm:DP-Selection}, but these can be safely removed with no loss in parameters since we are in the approximate DP setting.} Since the samples are independent, the probability of this occurring is at most
    \[
    \sum\limits_{j=O(\log(1/\delta)/\varepsilon)}^T{T \choose j}\beta^{j} \leq T\cdot \left(\frac{2\beta}{\eta}\right)^{O(\log(1/\delta)/\varepsilon)},
    \]
    as desired.

    It is left to bound the (user-level) differential privacy and support. It is clear the algorithm is $T$-user $(\varepsilon,\delta)$-DP by construction, taking each user's data to be a full sample given to $\mathcal{A}_{\text{List}}$ in Step 2. Furthermore, because the generated dataset always has a hypothesis appearing more than $\frac{c\log\frac{1}{\delta}}{\varepsilon}$ times by construction, \Cref{thm:DP-Selection} promises the output lies in the constructed dataset which has size at most $T+1$ as desired.
\end{proof}
The forward direction of both \Cref{thm:DP-list} and \Cref{thm:list-DP-user} are now essentially immediate:
\begin{proof}[Proof of \Cref{thm:DP-list} and \Cref{thm:list-DP-user} (Item 1)]
    Write $C=C_{Glob}(n',\beta')$ for notational simplicity, and let $\mathcal{A}$ be the promised $2^{-C+1}$-globally stable algorithm. By \Cref{lemma:list-to-DP}, we can convert $\mathcal{A}$ into a $T$-user $(\varepsilon,\delta/2)$-DP algorithm $\mathcal{A}_{DP}$ on $n$ samples with $T \leq O\left(2^{C} \cdot \frac{\log(1/\delta)}{\varepsilon}\right)$ such that for all $S$, $|\text{Supp}(\mathcal{A}(S;\cdot))| \leq T$, and with probability at least $1-\beta$, $\text{Supp}(\mathcal{A}_{DP}(S;\cdot)) \subset G_D$. Applying \Cref{lemma:less-random} with $\eta=\delta/2$ then gives the desired result, where correctness is maintained since the output of the algorithm on any sample $S$ where $\text{Supp}(\mathcal{A}_{DP}(S;\cdot)) \subset G_D$ remains entirely inside $G_D$ by the subset support property.
\end{proof}
We now move on to the reverse direction of both results, which follows from the equivalence of approximate differential privacy with another strong notion of stability known as \textit{perfect generalization} \cite{CummingsLNRW16,BassilyF16,bun2023stability,ghazi2024user}. 
\begin{definition}[Perfect Generalization (\cite{CummingsLNRW16,BassilyF16})]\label{def:perfect}
Fix $\beta,\varepsilon,\delta>0$. An algorithm $\mathcal{A}: \mathcal{X}^n \to \mathcal{Y}$ is called $(\beta,\varepsilon,\delta)$-perfectly-generalizing if for any distribution $D$ over $\mathcal{X}$, there exists a `canonical distribution' $\text{SIM}_D$ s.t.
\[
\Pr_{S \sim D^n}[\mathcal{A}(S) \overset{(\varepsilon,\delta)}{=} \text{SIM}_D] \geq 1-\beta.
\]
\end{definition}
In \cite{ghazi2024user}, the authors show any sufficiently DP algorithm is automatically perfectly generalizing.
\begin{theorem}[{\cite[Theorem 31]{ghazi2024user}} (rephrased)]\label{thm:DP-to-PG}
    There exists a universal constant $c>0$ such that if $\mathcal{A}$ is $T$-user $(\frac{c}{\sqrt{T\log(T)}},\frac{c}{T})$-DP, then $\mathcal{A}$ is $(.5,.5,.5)$-perfectly generalizing. Moreover $\text{SIM}_D$ can be taken to be the distribution $\mathcal{A}(\cdot,\cdot)$, taken over samples and internal randomness.
\end{theorem}
We remark that as written \cite[Theorem 31]{ghazi2024user} is only stated for standard (item-level) DP, but the result follows immediately from viewing the user-level DP guarantee as standard DP over input space $\mathcal{X}^{n/T}$. Combined with \Cref{thm:old-to-new-global}, it is then elementary to move from DP to global stability:
\begin{proof}[Proof of \Cref{thm:DP-list} and \Cref{thm:list-DP-user} (Item 2)]
    The proof is exactly the same for \Cref{thm:DP-list} and \Cref{thm:list-DP-user}, with the exception that we do not need to handle sample and confidence decay in the latter. We therefore argue just the former.
    
    By assumption, we are given a $\beta'$-confident $(\varepsilon,\delta)$-DP algorithm $\mathcal{A}$ on $\ell=C_{DP}(n',\beta',\varepsilon,\delta)$ random bits and $n'$ samples whose privacy parameters meet the requirements of \Cref{thm:DP-to-PG}. $\mathcal{A}$ is therefore $(.5,.5,.5)$-perfectly generalizing with respect to $\mathcal{A}(\cdot,\cdot)$. Fix an input sample $S$ where $\mathcal{A}(S,\cdot)$ is $(.5,.5)$-close to $\mathcal{A}(\cdot,\cdot)$. Write $L=\text{Supp}(\mathcal{A}(S,\cdot))$, and observe we have the trivial bound $|L| \leq 2^{\ell}$. On the other hand, distributional closeness implies
    \[
    1=\Pr[\mathcal{A}(S,\cdot) \in L] \leq e^{1/2}\Pr_{S \sim D^n,r}[\mathcal{A}(S,r) \in L] + \frac{1}{2},
    \]
    and re-arranging, that
    \[
    \Pr_{S \sim D^n,r}[\mathcal{A}(S,r) \in L] \geq \frac{1}{2e^{1/2}}.
    \]
    Thus $\mathcal{A}$ has a $\frac{1}{2|L|e^{1/2}}$-heavy-hitter. We may now apply our parametrized heavy-hitter to globally stable conversion (\Cref{cor:parametrized-old-to-new}) to convert $\mathcal{A}$ to a $\beta'$-confident deterministic globally stable algorithm on $n=n'\cdot \tilde{O}(\frac{\log(1/\beta)}{|L|^2})$ samples for $\beta \leq O(\beta' |L|$) as desired.
    % . Finally, we can easily transform $\mathcal{A}$ into an $O(L)$-list-replicable algorithm by running $\mathcal{A}$ $O(L\log(1/\beta))$ times and outputting the empirically most frequent hypothesis. Namely by standard concentration bounds, the probability a hypothesis with true frequency less than $O(1/L)$ occurs as the mode is at most $\beta$, meaning with probability $1-\beta$ the output lies in the at most $O(L)$ heavy hitters of $\mathcal{A}$ as desired.
\end{proof}
\section{The Stable Complexity of Agnostic Learning}\label{sec:PAC-cert}
In this section we bound the certificate complexity of PAC-Learning, resolving a main open question of \cite{chase2024local} (stated in the equivalent language of global-stability/list-replicability). We restate the theorem here for convenience.

\begin{theorem}[Certificate Complexity of Agnostic Learning]\label{thm:stable-agnostic}
    Let $(X,H)$ be a hypothesis class of Littlestone dimension $d$. There exists a $\beta$-confident better than $\frac{1}{2}$-replicable learner with
    \begin{enumerate}
        \item Sample Complexity: $m(\alpha,\beta) \leq \exp(\poly(d))\poly(\alpha,\log(1/\beta))$
        \item Certificate Complexity: $C_{\text{Rep}} \leq \poly(d)+O(VC(H)\log\frac{d}{\alpha})$
    \end{enumerate}
Moreover, if $d=\infty$, $C_{\text{Rep}}=\infty$, i.e.\ there is no replicable or globally stable learner for $H$.
\end{theorem}

Our proof is inspired by \cite{bun2023stability}'s version of \cite{hopkins2021realizable}'s realizable-to-agnostic reduction for replicable learners. \cite{bun2023stability} cannot be used directly due to its substantial reliance on additional randomness (indeed they actually have randoness scaling not only with confidence $\beta$, but with the size of the concept class $H$ which may be infinite in our case). Nevertheless, we build off the core idea, which uses the following realizable-case `list-stable' learner of  \cite{ghazi2021sample,GhaziKM21}:\footnote{We remark this is not quite as stated in \cite{GhaziKM21}, which only claims exponential dependence on $\alpha$. However, polynomial dependence as stated here is immediate from \cite{ghazi2021sample}}
\begin{theorem}[Realizable List-Stable Learning {\cite{ghazi2021sample,GhaziKM21}}]\label{thm:list-realizable}
    Let $(X,H)$ be a hypothesis class of Littlestone dimension $d$. There exists a algorithm on $n(\alpha,\beta) \leq \poly(d,\alpha,\log(1/\beta))$ samples which outputs a list of $\alpha$-accurate hypotheses of size at most $\exp(\poly(d))\alpha^{-O(1)}$ with probability at least $1-\beta$. Moreover, for every distribution $D$, the list has a $\Omega(1/d)$-heavy-hitter, i.e.\ some hypothesis $h_D$ such that
    \[
    \Pr_{S}[h_D \in \mathcal{A}(S)] \geq \Omega\left(\frac{1}{d}\right).
    \]
\end{theorem}

Given the above, the core idea is to run this list-stable learner over all possible labelings of a large enough unlabeled data sample and return a commonly appearing element with high accuracy over the agnostic distribution. This results in a randomized learner with a good heavy hitter, which can then be converted to a globally stable learner by \Cref{thm:old-to-new-global} and finally a replicable learner with good certificate complexity by \Cref{thm:equiv}.

\begin{proof}[Proof of \Cref{thm:stable-agnostic}]
    Fix $\beta' = \exp(\poly(d))\alpha^{-O(1)}$ to be sufficiently smaller than the list size from \Cref{thm:list-realizable} and let $\mathcal{A}$ be the corresponding promised learner. Consider the following process generating a list of potential hypotheses from $\mathcal{A}$:
    \begin{enumerate}
        \item Draw $T=\poly(d,\log(1/\beta))$ unlabeled samples of size $n(\alpha/8,\beta')$ from $D$, denoted $S^{(i)}_U$.
        \item Run $\mathcal{A}$ over all labelings of each $S^{(i)}_U$ in the class and denote the resulting hypothesis (multi)set as
        \[
        C(S_U) \coloneqq \{ \mathcal{A}(S^{(i)}_U,h(S^{(i)}_U))~:~ h \in H, i \in [T]\}
        \]
    \end{enumerate}   
Note that by \Cref{lemma:VC}, the size of $C(S_U)$ as a multiset is at most 
\[
|C(S_U)| \leq Tn(\alpha/8,\beta')^{O(VC(H))}\exp(\poly(d)) \leq T\poly(d,\alpha^{-O(1)})^{VC(H)}\exp(\poly(d)).
\]
The idea is now that $C(S_U)$ should contain many copies of a heavy hitter from the output list of some optimal labeling $h_{OPT}$ (which furthermore must be $\alpha/2$-correct by our confidence assumption). The only issues are 1) $C(S_U)$ may also contain many bad hypotheses, and 2) $|C(S_U)|$ currently depends on $T$ (which depends on the confidence $\beta$), so is too large for our purposes. We can fix both of these problems via a pruning step that removes any high error or non-heavy-hitter hypotheses. Denote by $w(h)$ the number of times $h$ appears in $C(S_U)$, and consider the following procedure:
\begin{enumerate}
        \item Draw a new labeled sample $S_L$ of $\frac{\log |C(S_U)| + \log \frac{1}{\beta}}{\alpha^2}$ samples
        \item Remove any hypothesis $h \in C(S_U)$ with high error or low empirical probability:
        \[
        \text{Pruned}(C(S_U)) \coloneqq \left\{h \in C(S_U): err_{S_L}(h) \leq \min_{h' \in C(S_U)} err_{S_L}(h') + \frac{3}{4}\alpha~\text{and}~ w(h) \geq c'\frac{T}{d}\right\}
        \]
        \item Output a random hypothesis in $\text{Pruned}(C(S_U))$
\end{enumerate}
where $c'>0$ is some constant to be chosen. Note that now, by construction, $\text{Pruned}(C(S_U))$ has at most $\poly(d/\alpha)^{VC(H)}\exp(\poly(d))$ hypotheses since at most this many can appear $\Omega(T/d)$ times.

We now argue correctness and (heavy-hitter) global stability of the procedure (certificate complexity will then follow from \Cref{thm:old-to-new-global} and \Cref{thm:equiv}). By standard uniform convergence bounds, we have that with probability at least $1-\beta/2$ the empirical and true error of every $h \in C(S_U)$ are close:
\[
|err_D(h)-err_{S_L}(h)| \leq \frac{\alpha}{8}.
\]
Conditioned on the above, correctness follows as long as $C(S_U)$ actually contains a good hypothesis:
\[
\min_{h' \in C(S_U)} err_{S_L}(h') \leq \text{OPT}_D+\frac{\alpha}{8}
\]
since pruning then only keeps hypotheses with empirical error at most $\text{OPT}+\frac{7}{8}\alpha$ and therefore of true error at most $\alpha$ as desired.

In fact we will argue not just that this event occurs with probability at least $1-\beta/2$, but that there exists a \textit{fixed} $\alpha/8$-optimal hypothesis $h_D$ that appears with this probability. Since our final algorithm outputs a random hypothesis in $\text{Pruned}(C(S_U))$, this implies both our desired accuracy and global stability. To see this, fix some optimal hypothesis $h_{OPT}$ in the class, and observe that the original set $C(S_U)$ contains $T$ independent runs of $\mathcal{A}$ on samples labeled by $h_{OPT}$ (since we run over all labelings). $\mathcal{A}$ is promised to have an $\alpha/8$-optimal $\Omega(1/d)$-heavy-hitter $h_D$ with respect to the distribution $D_X \times h_{OPT}$ where the marginal over data is given by $D$ and the labeling is given by $h_{OPT}$. As a result $h_D$ itself is $\frac{1}{8}\alpha$-optimal with respect to the true distribution $D$. Moreover, since $T=\poly(d,\log(1/\beta))$, a standard Chernoff bound implies $h_D$ appears at least $c'T/d$ times with probability at least $1-\beta/2$ for the appropriate choice of constant $c'>0$, completing the argument.

The above analysis promises the given algorithm is $\beta$-correct and has an $\Omega(\exp(\poly(d))\alpha^{-O(1)})$-heavy hitter. We can convert this into the desired replicable learner using \Cref{thm:old-to-new-global} and \Cref{thm:equiv} with sample overhead at worst polynomial in the stability parameter as desired.

    It is left to show that when $d=\infty$, there is no replicable (equivalently, globally stable) learner with better than $\frac{1}{2}$ error. To see this, first observe that a replicable agnostic learner with better than $\frac{1}{2}$ error implies a replicable learner in the realizable case with better than $\frac{1}{2}$ error. By \cite{ImpLPS22}, this can be boosted to a replicable realizable learner with arbitrary error $\alpha$, then transformed to a differentially private one by standard stability-to-privacy reductions (e.g.\ \cite{bun2020equivalence} or the transforms presented in this work). The existence of this final learner then violates the impossibility of private realizable learning of infinite Littlestone classes \cite{alon2019private}, completing the proof.
\end{proof}
\section*{Acknowledgements}
We thank Zachary Chase for many helpful discussions on list stability and differential privacy, especially for referring us to \cite{ghazi2021sample}'s list-stability bound with polynomial dependence in excess error in the realizable regime, and suggesting how to make list-replicable learners deterministic.

\bibliographystyle{amsalpha}  
\bibliography{references} 
\appendix
\section{Amplifying Replicability}\label{app:amplify}
A particularly useful property of replicability is that it can be easily amplified. In the main body, we implicitly gave a randomness-efficient procedure to amplify better than $\frac{1}{2}$ replicability to any $\rho$ based on converting to and from global stability. Here, we give a randomness inefficient but sample-efficient transform with the added benefit of starting from any starting replicability parameter $\nu$. The transform is similar to that in \cite{kalavasis2023statistical} for so-called `TV-indistinguishable' algorithms, but has improved sample complexity with respect to $\nu$.
\begin{lemma}[Generic Amplification of Replicability]
    Fix $\nu,\beta>0$, and let $\mathcal{A}$ be an $(1-\nu)$-replicable, $\beta$-correct algorithm on $n=n(\nu,\beta)$ samples. For any $\rho>0$, there exists an efficient blackbox procedure amplifying $\mathcal{A}$ to a $\rho$-replicable, $\beta'$-correct algorithm $A'$ on $\tilde{O}\left(n\cdot \frac{\log(1/\beta)}{\rho^2\nu^2}\right)$ samples for $\beta' \leq \tilde{O}\left(\beta\frac{\log(1/\beta)}{\rho^2\nu^3}\right)$.
\end{lemma}
\begin{proof}
    Since $\mathcal{A}$ is $(1-\nu)$-replicable, by Markov's inequality there must be at least a $\nu/2$-fraction of $\mathcal{A}$'s random strings for which the distribution $\mathcal{A};r)$ has an $O(\nu/2)$-heavy-hitter. Draw a set $T$ of $O\left(\frac{\log(1/\rho)}{\nu}\right)$ random strings. The probability that at least one drawn string has a $O(\nu/2)$-heavy-hitter is at least $1-(1-\nu/2)^{O(\log(1/\rho)/\nu)} \geq 1-\rho/2$.

    Condition on this event and consider the algorithm which, given a sample $S$, outputs a list containing the result of $\mathcal{A}$ on $S$ across all random strings in $T$:
    \[
    M(S) \coloneqq \{ \mathcal{A}(S;r) : r \in T\}
    \]
    By assumption, $M$ satisfies the following properties, together called ``list-stability'' \cite{GhaziKM21,bun2023stability}:
    \begin{enumerate}
        \item $M(S)$ Outputs a list of size $|T|$
        \item $\exists h$ s.t.\ $h \in M(S)$ with probability at least $O(\nu)$.
        \item With probability at least $\beta'=|T|\beta$, all $h \in T$ are correct.
    \end{enumerate}
    We now appeal to \cite[Theorem 6.7]{bun2023stability}, which gives a $\rho/2$-replicable procedure that runs $M$ on
    \[
    m=\tilde{O}\left(\frac{\log(1/\beta)}{\rho^2\nu^2}\right)
    \]
    independent size-$n$ samples,\footnote{Here $\tilde{O}$ hides (at most cubic) logarithmic factors in $\rho$ and $\nu$.} and outputs some hypothesis lying in one the resulting $M(S)$ with probability at least $1-\beta$. By a union bound, all such hypotheses are correct except with probability $|T|m\beta$, and the entire procedure is $\rho$-replicable as desired.
\end{proof}

\end{document}